\documentclass{article}

\usepackage{microtype}
\usepackage{graphicx}
\usepackage{subfigure}
\usepackage{booktabs} % for professional tables

\usepackage{hyperref}

\usepackage[accepted]{icml2023}

\usepackage{amsmath}
\usepackage{amssymb}
\usepackage{mathtools}
\usepackage{amsthm}
\usepackage{amsmath,amsfonts,bm}

\def\eqref#1{equation~\ref{#1}}
\def\1{\bm{1}}

\def\R{{R}}

\def\va{{\bm{a}}}
\def\vb{{\bm{b}}}

\def\vd{{\bm{d}}}

\def\vg{{\bm{g}}}
\def\vh{{\bm{h}}}

\def\vk{{\bm{k}}}

\def\vm{{\bm{m}}}

\def\vp{{\bm{p}}}
\def\vq{{\bm{q}}}

\def\vs{{\bm{s}}}

\def\vu{{\bm{u}}}
\def\vv{{\bm{v}}}

\def\vx{{\bm{x}}}
\def\vy{{\bm{y}}}
\def\vz{{\bm{z}}}

\def\mI{{\bm{I}}}

\def\mM{{\bm{M}}}

\def\mO{{\bm{O}}}

\def\mT{{\bm{T}}}

\def\mW{{\bm{W}}}

\def\mZ{{\bm{Z}}}

\DeclareMathAlphabet{\mathsfit}{\encodingdefault}{\sfdefault}{m}{sl}
\SetMathAlphabet{\mathsfit}{bold}{\encodingdefault}{\sfdefault}{bx}{n}

\def\gE{{\mathcal{E}}}

\def\gG{{\mathcal{G}}}

\def\gV{{\mathcal{V}}}

\def\sR{{\mathbb{R}}}

\newcommand{\eg}{\emph{e.g.}}
\newcommand{\ie}{\emph{i.e.}}

\newcommand{\name}{\textsc{Set}}
\newcommand{\Solar}{\textsc{Solar}}
\newcommand{\Amorpheus}{\textsc{Amorpheus}}
\newcommand{\SMP}{{SMP}}
\newcommand{\SWAT}{{SWAT}}
\newcommand{\Monolithic}{Monolithic}
\usepackage{multirow} 
\usepackage[capitalize,noabbrev]{cleveref}
\newcommand{\tabincell}[2]{\begin{tabular}{@{}#1@{}}#2\end{tabular}}  
\theoremstyle{plain}
\newtheorem{theorem}{Theorem}[section]

\newtheorem{corollary}[theorem]{Corollary}
\theoremstyle{definition}
\newtheorem{definition}[theorem]{Definition}

\theoremstyle{remark}

\usepackage[textsize=tiny]{todonotes}

\icmltitlerunning{Subequivariant Graph Reinforcement Learning in 3D Environments}

\begin{document}

\twocolumn[
\icmltitle{Subequivariant Graph Reinforcement Learning in 3D Environments}



\icmlsetsymbol{equal}{*}

\begin{icmlauthorlist}
\icmlauthor{Runfa Chen}{equal,thu}
\icmlauthor{Jiaqi Han}{equal,thu}
\icmlauthor{Fuchun Sun}{thu,bosch}
\icmlauthor{Wenbing Huang}{ruc,bdmam}
\end{icmlauthorlist}

\icmlaffiliation{thu}{Dept. of Comp. Sci. \& Tech., Institute for AI, BNRist Center, Tsinghua University}
\icmlaffiliation{bosch}{THU-Bosch JCML Center} 
\icmlaffiliation{ruc}{Gaoling School of Artificial Intelligence, Renmin University of China}
\icmlaffiliation{bdmam}{Beijing Key Laboratory of Big Data Management and Analysis Methods}

% \icmlcorrespondingauthor{Fuchun Sun}{fcsun@mail.tsinghua.edu.cn}
% \icmlcorrespondingauthor{Wenbing Huang}{hwenbing@126.com}

\icmlkeywords{Machine Learning, ICML}

\vskip 0.3in
]



\printAffiliationsAndNotice{\icmlEqualContribution} % otherwise use the standard text.

\begin{abstract}
Learning a shared policy that guides the locomotion of different agents is of core interest in Reinforcement Learning (RL), which leads to the study of morphology-agnostic RL. However, existing benchmarks are highly restrictive in the choice of starting point and target point, constraining the movement of the agents within 2D space. In this work, we propose a novel setup for morphology-agnostic RL, dubbed Subequivariant Graph RL in 3D environments (3D-SGRL). Specifically, we first introduce a new set of more practical yet challenging benchmarks in 3D space that allows the agent to have full Degree-of-Freedoms to explore in arbitrary directions starting from arbitrary configurations. Moreover, to optimize the policy over the enlarged state-action space, we propose to inject geometric symmetry, \ie, subequivariance, into the modeling of the policy and Q-function such that the policy can generalize to all directions, improving exploration efficiency. This goal is achieved by a novel SubEquivariant Transformer (\name{}) that permits expressive message exchange. Finally, we evaluate the proposed method on the proposed benchmarks, where our method consistently and significantly outperforms existing approaches on single-task, multi-task, and zero-shot generalization scenarios. Extensive ablations are also conducted to verify our design.
\end{abstract}

\section{Introduction}

\begin{figure}[t!]
\centering
\includegraphics[width=\linewidth]{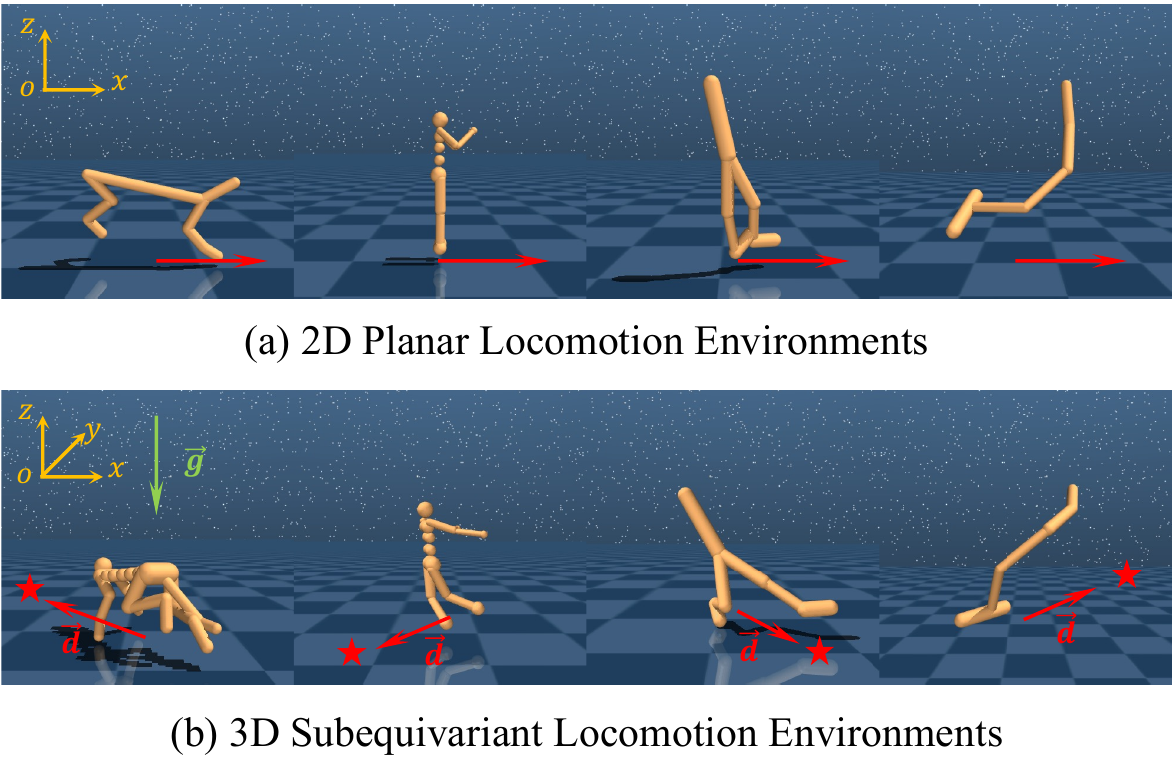}
\vspace{-.32in}
\caption{Illustrative comparison between previous 2D planar setting and our 3D subequivariant formulation. Notably, the agents in (b) are equipped with more DoFs to allow 3D movement.  Code and videos are available on our project page: \href{https://alpc91.github.io/SGRL/}{https://alpc91.github.io/SGRL/}.}
\label{fig:teaser}
\vspace{-.2in}
\end{figure}

Learning to locomote, navigate, and explore in the 3D world is a fundamental task in the pathway of building intelligent agents. Impressive breakthrough has been made towards realizing such intelligence thanks to the emergence of deep reinforcement learning (RL)~\cite{mnih2015human,silver2016mastering,mnih2016asynchronous,schulman2017proximal,fujimoto2018addressing}, where the policy of the agent is acquired through interactions with the environment. More recently, by getting insight into the morphology of the agent, morphology-agnostic RL~\cite{wang2018nervenet,pathak2019learning,huang2020one,kurin2020cage,hong2021structure,dong2022low,trabucco2022anymorph,gupta2022metamorph,furuta2023asystem} has been proposed with the paradigm of learning a local and shared policy for all agents and the tasks involved, offering enhanced performance and transferability, especially in the multi-task scenario. It is usually fulfilled by leveraging Graph Neural Networks (GNNs)~\cite{battaglia2018relational} or even Transformers~\cite{vaswani2017attention} to derive the policy through passing and fusing the state information on the morphological graphs of the agents.

In spite of the fruitful progress by morphology-agnostic RL, in this work, we identify several critical setups that have been over-simplified in existing benchmarks, giving rise to a limited state/action space such that the obtained policy is unable to explore the entire 3D space. In particular, the agents are assigned a fixed starting point and restricted to moving towards a single direction along the $x$-axis, leading to 2D motions only. Nevertheless, in a more realistic setup as depicted in \Cref{fig:teaser}, the agents would be expected to have full Degree-of-Freedoms (DoFs) to turn and move in arbitrary directions starting from arbitrary configurations. To address the concern, we extend the existing environments to a set of new benchmarks in 3D space, which meanwhile introduces significant challenges to morphology-agnostic RL due to the massive enlargement of the state-action space for policy optimization. 

Optimizing the policy in our new setup is prohibitively difficult, and existing morphology-agnostic RL frameworks like~\citep{huang2020one,hong2021structure} are observed to be susceptible to getting stuck in the local minima and exhibited poor generalization in our experiments. To this end, we propose to inject geometric symmetry~\cite{cohen2016group, cohen2016steerable, worrall2017harmonic,van2020mdp} into the design of the policy network to compact the space redundancy in a lossless way~\cite{van2020mdp}. In particular, we restrict the policy network to be subequivariant in two senses~\cite{han2022learning}: 1. the output action will rotate in the same way as the input state of the agent; 2. the equivariance is partially relaxed to take into account the effect of gravity in the environment. We design SubEquivariant Transformer (\name{}) with a novel architecture that satisfies the above constraints while also permitting expressive message propagation through self-attention. Upon \name{}, the action and Q-function could be obtained with desirable symmetries guaranteed. We term our entire task setup and methodology as Subequivariant Graph Reinforcement Learning in 3D Environments (3D-SGRL).

Our contributions are summarized as follows:
\begin{itemize}
    \vspace{-.1in}
    \item We introduce a set of more practical yet highly challenging benchmarks for morphology-agnostic RL, where the agents are permitted to turn and move in the 3D environments with arbitrary starting configurations and arbitrary target directions. For this purpose, we redesign the agents in current benchmarks by equipping them with more DoFs in a considerate way.  
    \vspace{-.1in}
    \item To effectively optimize the policy on such challenging benchmarks, we propose to enforce the policy network with geometric symmetry. We introduce a novel architecture dubbed \name{} that captures the rotation/translation equivariance particularly when external force fields like gravity exist in the environment.
    \vspace{-.1in}
    \item We verify the performance of the proposed method on the proposed 3D benchmarks, where it outperforms existing morphology-agnostic RL approaches by a significant margin in various scenarios, including single-task, multi-task, and zero-shot generalization. Extensive ablations also reveal the efficacy of the proposed ideas.
    \vspace{-.1in}
\end{itemize}

\section{Background}\label{sec:background}

\paragraph{Morphology-Agnostic RL} In the context of morphology-agnostic RL~\citep{huang2020one}, we are interested in an environment with $N$ agents (\emph{a.k.a} tasks), where the $n$-th agent comprises $K_n$ limbs that control its motion. At time $t$, each limb $k\in\{1,\cdots,K_n\}$ of agent $n$ receives a state $\vs_{n,k}(t)\in \sR^d$ and outputs a torque $a_{n,k}(t) \in [-1,1]$ to its actuator. As a whole, agent $n$ executes the joint action $\va_n(t)=\{a_{n,k}(t)\}_{k=1}^{K_n}$ to interact with the environment which will return the next state of all limbs $\vs_n(t+1)=\{\vs_{n,k}(t+1)\}_{k=1}^{K_n}$ and a reward $r_n(\vs_n(t),\va_n(t))$ for agent $n$. The goal of morphology-agnostic RL is to learn a shared policy $\pi_\theta$ among different agents to maximize the expected return:
\begin{equation}
    \mathcal{J}(\theta)=\mathbb{E}_{\pi_\theta} \sum_{n=1}^{N}\sum_{t=0}^\infty\left[\gamma^t r_{n}(\vs_n(t),\va_n(t))\right],
    \label{eq:policy}
\end{equation}
where $\va_{n}(t)=\pi_\theta(\vs_{n}(t))$, $\gamma$ is a discount factor, and $\theta$ consists of trainable parameters.

The objective in \Cref{eq:policy} is usually optimized via the actor-critic setup of the deterministic policy gradient algorithm for continuous control~\cite{lillicrap2015continuous}, which estimates the Q-function for agent $n$:
\begin{equation}
\label{eq:Q}
\aligned
    Q_{\pi_\theta}(\vs_{n},\va_{n})=\mathbb{E}_{\pi_\theta} \sum_{t=0}^\infty&[\gamma^t r_n(\vs_n(t),\va_n(t))|\\
    &\vs_{n}(0)=\vs_{n},\va_{n}(0)=\va_{n}].
\endaligned
\end{equation}

To uniformly learn a shared policy across all agents and tasks, previous methods~\citep{wang2018nervenet,pathak2019learning,huang2020one,kurin2020cage,hong2021structure,dong2022low},  
take into account the interaction of connected limbs and joints, and view the morphological structure of the agent as an undirected graph $\gG=(\gV,\gE)$, where each $v_i\in\gV$ represents a limb and the edge $(v_i, v_j)\in\gE$ stands for the joint connecting limb $i$ and $j$\footnote{For simplicity, we omit the index $n$ and $t$ henceforth in the above notations of agent $n$ at time $t$, since all agents share the same model for all time, \eg, $\vs_{n}(t)\rightarrow\vs$ and $\va_{n}(t)\rightarrow\va$.}. A graph neural network $\varphi_\theta$ is then employed to instantiate the policy $\pi_\theta$, which predicts the action $\va$ given the state of all limbs $\vs$ and the graph topology $\gE$ as input, \ie,
 \begin{align}
 \label{eq:graphrl}
     \va = \varphi_\theta\left(\vs, \gE \right).
 \end{align}

\paragraph{Equivariance and Subequivariance}
To further relieve the difficulty of learning a desirable policy within the massive search space formed by the states and actions of the agent in 3D space, we propose to encode the physical geometric symmetry of the policy learner $\varphi_\theta$, so that the learned policy can generalize to operations in 3D, including rotations, translations, and reflections, altogether forming the group of E($3$). Such constraint enforced on the model is formally described by the concept of \emph{equivariance}~\cite{thomas2018tensor,fuchs2020se3,villar2021scalars,satorras2021en,huang2022equivariant,han2022learning,han2022geometrically}.

\begin{definition}[E($3$)-equivariance]
    Suppose $\Vec{\mZ}$ to be 3D geometric vectors (positions, velocities, etc) that are steerable by E(3) transformations, and $\vh$  non-steerable features.
    The function $f$ is E($3$)-equivariant, if for any transformation $g\in\text{E}(3)$, $f(g\cdot\Vec{\mZ},\vh)=g\cdot f(\Vec{\mZ},\vh)$, $\forall\Vec{\mZ}\in\sR^{3\times m}, \vh\in\sR^{d}$.  Similarly, $f$ is invariant if $f(g\cdot\Vec{\mZ},\vh)= f(\Vec{\mZ},\vh)$. 
\end{definition}

Built on this notion,~\citet{han2022learning} additionally considers equivariance on the subgroup of $\text{O}(3)$, induced by the external force $\vec{\vg}\in\sR^3$ like gravity, defined as $\text{O}_{\vec\vg}(3)\coloneqq\{\mO\in\sR^{3\times 3}|\mO^\top\mO=\mI, \mO\vec\vg=\vec\vg  \}$. By this means, the symmetry is only restrained to the rotations/reflections along the direction of $\vec\vg$. Such relaxation of group constraint is crucial in environments with gravity, as it offers extra flexibility to the model so that the effect of gravity could be captured.~\citet{han2022learning} also presented a universally expressive construction of the $\text{O}_{\vec\vg}(3)$-equivariant functions:
\begin{align}
\label{eq:subGMN}
\begin{aligned}
    &f_{\vec{\vg}}(\Vec{\mZ},\vh)=[\Vec{\mZ}, \vec{\vg}]\mM_{\vec{\vg}},\\
    &\text{s.t.}\quad \mM_{\vec{\vg}}=  \sigma([\Vec{\mZ},\vec{\vg}]^{\top}[\Vec{\mZ},\vec{\vg}],\vh),
\end{aligned}
\end{align}
where $\sigma\left(\cdot\right)$ is an Multi-Layer Perceptron (MLP) and $[\Vec{\mZ}, \vec{\vg}]\in\sR^{3\times (m+1)}$ is a stack of $\vec\mZ$ and $\vec\vg$ along the last dimension. In particular, $f$ will reduce to be $\text{O}(3)$-equivariant if $\vec\vg$ is omitted in the computation. In this way, $f_{\vec\vg}$ can then be leveraged in the message passing process of the graph neural network $\varphi_\theta$ in \Cref{eq:graphrl} to obtain desirable geometric symmetry. 
\section{Our task and method: 3D-SGRL}\label{sec:method}

In this section, we present our novel formulation for morphology-agnostic RL, dubbed Subequivariant Graph Reinforcement Learning in 3D Environments (3D-SGRL). We first elaborate on the extensions made to the environment in \Cref{sec3.1}, then introduce our entire framework, consisting of an input processing module (\Cref{sec3.2}), a novel SubEquivariant Transformer (\name{}) for expressive information passing and fusion (\Cref{sec3.3}), and output modules of actor and critic to obtain the final policy and Q-function (\Cref{sec3.4}).

\begin{figure*}[t!]
\centering
\includegraphics[width=\linewidth]{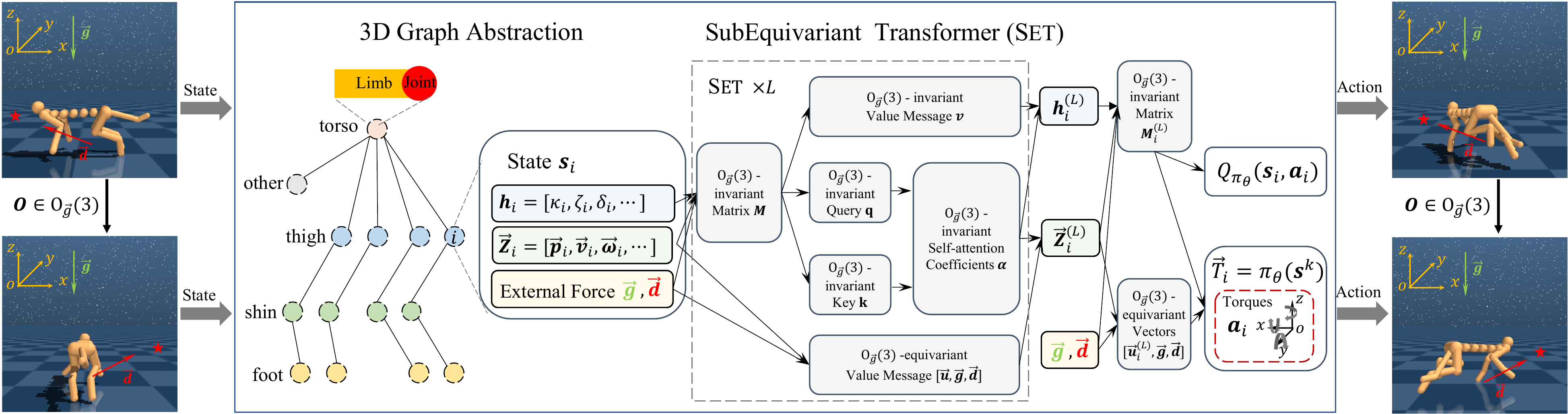}
\vspace{-.32in}
\caption{The flowchart of our 3D-SGRL. The states of the agents are processed into $\vh_i$ and $\vec{\mZ}_i$ for each limb $i$, and are updated by $L$ layers of our proposed SubEquivariant Transformer. The actor and critic are finally obtained, which are guaranteed to preserve the geometric symmetry for guiding the agent in arbitrary directions. There is no weight sharing between actor $\pi_\theta$ and critic $Q_{\pi_\theta}$.}
\label{fig:main}
\vspace{-.15in}
\end{figure*}

\begin{table}[htbp]
\setlength\tabcolsep{3pt}
    \centering
    \small
    \caption{Comparison in the problem setup.}
    \label{tab:environment_vs}
    \begin{tabular}{lllll}
    \toprule
    & &\textbf{2D-Planar}&\textbf{Our 3D-SGRL}\\
     \midrule
     \multirow{3}{*}{\tabincell{l}{State Space}} & Range & $xoz$-plane & 3D space\\
     & Initial  &  $x^+$-axis & Arbitrary direction \\
     & Target  & $x^+$-axis & Arbitrary direction\\
    \midrule
    \multirow{2}{*}{\tabincell{l}{Action Space}}&  \# Actuators  & 1 per joint & 3 per joint\\
     & DoF & 1 per joint & 3 per joint\\
     \midrule
    \multirow{2}{*}{\tabincell{l}{Symmetry}}&  External Force  & NULL & Gravity $\vec{\vg}$, Target $\vec{\vd}$ \\
     & Group & $\emptyset$ &  $\text{O}_{\vec{\vg}}$(3)\\
    \bottomrule
    \end{tabular}
    \label{tab:compare}
\end{table}

\subsection{From 2D-Planar to 3D-SGRL}
\label{sec3.1}

A core mission of developing RL algorithms is enabling the agent (\eg, a robot) to learn to move in the environment with a designated goal. Ideally, the exploration should happen in the open space where the agent is able to move from the arbitrary starting point, via arbitrary direction, towards an arbitrary destination, offering much flexibility which highly corresponds to how the robot walks/runs in the real world. However, in the widely acknowledged setup in existing morphology-agnostic RL literature~\cite{huang2020one,kurin2020cage,hong2021structure,dong2022low}, the agents are unanimously restricted in the fixed choice of starting position, target direction, and even the Degree-of-Freedom (DoF) of each joint in the action space. We summarize the limitations of the existing setup, which we dub \emph{2D-Planar}, and compare it with our introduced 3D-SGRL in \Cref{tab:compare} in three aspects, including state space, action space, and the consideration of geometric symmetry.

\paragraph{State Space} In the 2D-Planar setup, all positions of the limbs are projected onto the $xoz$-plane, and the agent is always initialized to face the positive $x$-axis. The agent is also designated to move in the same direction as it is initialized, lacking many vital movements, \eg, turning, that an agent is supposed to learn. In our 3D-SGRL environment, all agents are initialized randomly in the full 3D space, facing a random direction, with the goal of moving towards a random destination. This setup is more like a comprehensive navigation task, which brings significant challenges by permitting an input/output state space with much higher complexity.

\paragraph{Action Space} For a more detailed granularity, our 3D-SGRL also expands the action space that offers the agent more flexibility to explore and optimize the policy on this challenging task. Specifically, the number of actuators is increased from only 1 on each joint in 2D-Planar to 3 per joint, which implies the DoF on each joint is also enlarged from 1 to 3 correspondingly.

\paragraph{Geometric symmetry} Since both the state space and action space have been enormously augmented, the functional complexity of the policy network $\varphi_\theta$ in \Cref{eq:graphrl} scales geometrically in correspondence. This poses a unique challenge, especially in RL, where the skills of the agent are gradually obtained through abundant explorations in the environments. During the learning process, the optimization of $\varphi_\theta$ becomes highly vulnerable to getting stuck in local minima, and searching for a good policy within the large space would be notoriously difficult. To tackle this challenge, we propose to take advantage of the geometric symmetry in the environments by enforcing it as a constraint in the design of $\varphi_\theta$. In particular, we construct $\varphi_\theta$ to be an $\text{O}_{\vec\vg}(3)$-equivariant function, which ensures that the policy learned in each direction can generalize seamlessly to arbitrary direction rotated along the gravity axis. Instead of $\text{O}(3)$, we resort to subequivariant $\text{O}_{\vec\vg}(3)$ to empower the model such that the effect of gravity reflecting in the policy can be well captured. By contrast, existing morphology-agnostic RL works lack the consideration of geometric symmetry, leading to poor performance in a real and more challenging setup like 3D-SGRL. In addition to gravity, we have a target direction $\vec{\vd}\in\sR^3$ that is steerable and acts like an attracted force guiding the agent towards expected destinations. The task guidance is not explicitly specified in the previous 2D-Planar setting but comes as an indispensable clue in our 3D-SGRL tasks. 

\subsection{Input Processing}\label{sec3.2}
To fulfill the constraint in geometric symmetry, we need to subdivide the state $\vs_i$ into the directional geometric vectors $\vec{\mZ}_i$ and the scalar features $\vh_i$ for each node $i\in\{1,\cdots,|\gV|\}$ in the morphological graph $\gG$ of the agent. Quantities in $\vec\mZ_i$ will rotate in accordance with the transformation $g\in\text{O}_{\vec\vg}(3)$ while those in $\vh_i$ remain unaffected. To be specific, for our 3D environments generated by MuJoCo~\cite{todorov2012mujoco}, the vectors in $\vec\mZ_i\in\R^{3 \times 6}$ include the position $\vec{\vp}_i\in\sR^3$, 
the positional velocity $\vec{\vv}_i\in\sR^3$, 
the rotational velocity $\vec{\bm{\omega}}_i\in\sR^3$, 
joint rotation $x$-axis $\vec{\vx}_i\in\sR^3$, 
joint rotation $y$-axis $\vec{\vy}_i\in\sR^3$, and
joint rotation $z$-axis $\vec{\vz}_i\in\sR^3$. 
The values in $\vh_i \in \sR^{13} $ consist of the rotation angles $\kappa_i$, $\zeta_i$,  $\delta_i$ of joint $x$-axis, $y$-axis, and $z$-axis, respectively, and their corresponding ranges as well as the type of limb, which is a 4-dimensional one-hot vector representing ``torso'', ``thigh'', ``shin'', ``foot'' and ``other'' respectively. As mentioned before, we have a target direction $\vec{\vd}$ apart from $\vec{\mZ}_i$ and $\vh_i$. Specifically,  $\vec{\vd}\coloneqq [\frac{\vec{\vp}^{xy}-\vec{\vp}_1^{xy}}{\Vert\vec{\vp}^{xy}-\vec{\vp}_1^{xy}\Vert_2},0]$, where $\vec{\vp}^{xy}$ is the $xy$ coordinate of the assigned target and $\vec{\vp}_1^{xy}$ is the $xy$ coordinate of limb 1 (torso), each of which is in $\sR^2$, and the resulting $\vec\vd \in \sR^3$.

\subsection{SubEquivariant Transformer (\name{})}\label{sec3.3}

Given the states encoded in $\vec{\mZ}_i$ and $\vh_i$, $i\in\{1,\cdots,|\gV|\}$, we are still in demand of a highly expressive $\varphi_\theta$ to learn the policy while ensuring the subequivariance. To this end, we present a novel architecture \name{}, to conduct effective message fusion between the limbs and joints, where the attention module is carefully designed to meet the symmetry. 

In particular, our \name{} processes the following operations in each computation.
\begin{align}
    \label{eq:z}
    \vh^{(0)}_i&=[\vh_i, \vec{\vp}_i^z], \\
    \vec{\mZ}^{(0)}_i &=  \vec{\mZ}_i\ominus\vec{\mZ}_1\coloneqq[\vec{\vp}_i-\vec{\vp}_1,\vec{\vv}_i,\vec{\bm{\omega}}_i,\vec{\vx}_i,\vec{\vy}_i,\vec{\vz}_i],
\end{align}
where, the binary operation ``$\ominus$'' transforms the input positions into translation invariant representations by subtracting $\vec{\vp}_1$, the position of the node with index $1$, \ie, the torso limb; $\vec{\vp}_i^z$ is the projection of the coordinate $\vec{\vp}_i$ to the $z$-axis, which is indeed the relative height of node $i$ when taking the ground as reference. The superscript $0$ indicates the processed input.

In the next step, we derive an $\text{O}_{\vec{\vg}}(3)$-invariant matrix $\mM_i \in \sR^{m \times m}$ as the value matrix in self-attention. Formally,
\begin{align}
\label{eq:M}
\mM^{(l)}_i=\sigma_\mM\left(\sigma_{\vec\vm}\left([\vec{\vm}^{(l)}_i,\vec\vg,\vec{\vd}]^\top[\vec{\vm}^{(l)}_i,\vec\vg,\vec{\vd}]\right),\vh^{(l)}_i\right),
\end{align}
where $\vec{\vm}^{(l)}_i=\vec{\mZ}^{(l)}_i \mW^{(l)}_{\vec{\vm}}$ is a mixing of the vectors in $\vec{\mZ}^{(l)}_i$ to capture the interactions between channels, with a learnable weight matrix $\mW^{(l)}_{\vec{\vm}}$; the concatenation with $\vec\vg$ and $\vec\vd$, and the inner product operation follow the practice in \Cref{eq:subGMN}; $\sigma_{\vec{\vm}}$ and $\sigma_\mM$ are two separate MLPs, and the superscript $l$ indexes the layer number.

With the value matrix $\mM_i$, we compute the self-attention coefficients $\alpha_{ij} \in \sR^{|\gV| \times |\gV|}$ between all pairs of node $i$ and $j$, by deriving the $\text{O}_{\vec{\vg}}(3)$-invariant query and key:
\begin{align}
    \vq_i^{(l)} &= \mW_q^{(l)}\text{vec}(\mM_i^{(l)})+\vb_q^{(l)},\\
    \vk_i^{(l)} &= \mW_k^{(l)}\text{vec}(\mM_i^{(l)})+\vb_k^{(l)},\\
    \alpha_{ij}^{(l)}&=\frac{\exp(\vq_i^{(l)\top}\vk^{(l)}_j)}{\sum_{m}\exp(\vq^{(l)\top}_i\vk^{(l)}_{m})},
\end{align}
where $\text{vec}(\cdot)$ is a column vectorization function of matrix: $\sR^{m\times m}\mapsto\sR^{mm \times 1}$, $\mW_q^{(l)},\mW_k^{(l)} \in \sR^{mm \times mm}$ are the learnable weights and $\vb^{(l)}_q,\vb^{(l)}_k \in \sR^{mm \times 1}$ are the biases in the $l$-th layer.

Finally, the $\text{O}_{\vec\vg}(3)$-equivariant and invariant values are transformed by the attention coefficients $\alpha_{ij}$ and aggregated to obtain the updated information. In detail,
\begin{align}
\label{eq:mz}
\vec{\mZ}^{(l+1)}_i&=\vec{\mZ}^{(l)}_i+\sum_j\left(\alpha_{ij}^{(l)}[\vec{\vu}_j^{(l)},\vec{\vg},\vec{\vd}]\right)\mW^{(l)}_{\vec{\mZ}},\\
\label{eq:h}
\vh^{(l+1)}_i&=\text{LN}\left(\vh^{(l)}_i+\mW^{(l)}_\vh\sum_j\left(\alpha^{(l)}_{ij}\vv^{(l)}_j\right)  + \vb^{(l)}_{\vh}\right),
\end{align}
where $\vec{\vu}^{(l)}_j=\vec{\mZ}^{(l)}_j \mW^{(l)}_{\vec{\vu}}$ is a mixing of the vectors in $\vec{\mZ}^{(l)}_j$ to capture the interactions between channels,  $\vv^{(l)}_j=\mW^{(l)}_{\vv} \text{vec}(\mM^{(l)}_j)+ \vb^{(l)}_{\vv}$ is a invariant value message, with learnable weight matrices $\mW^{(l)}_{\vec{\vu}}, \mW^{(l)}_{\vv}$ and the bias $\vb^{(l)}_{\vv}$, and $\text{LN}\left(\cdot\right)$ is the Layer Normalization~\cite{ba2016layer}. 

The operations are stacked over $L$ layers in total, resulting in the final architecture of~\name{}, with the full flowchart visualized in \Cref{fig:main}.

\begin{figure*}[t!]
\centering
\includegraphics[width=0.8\linewidth]{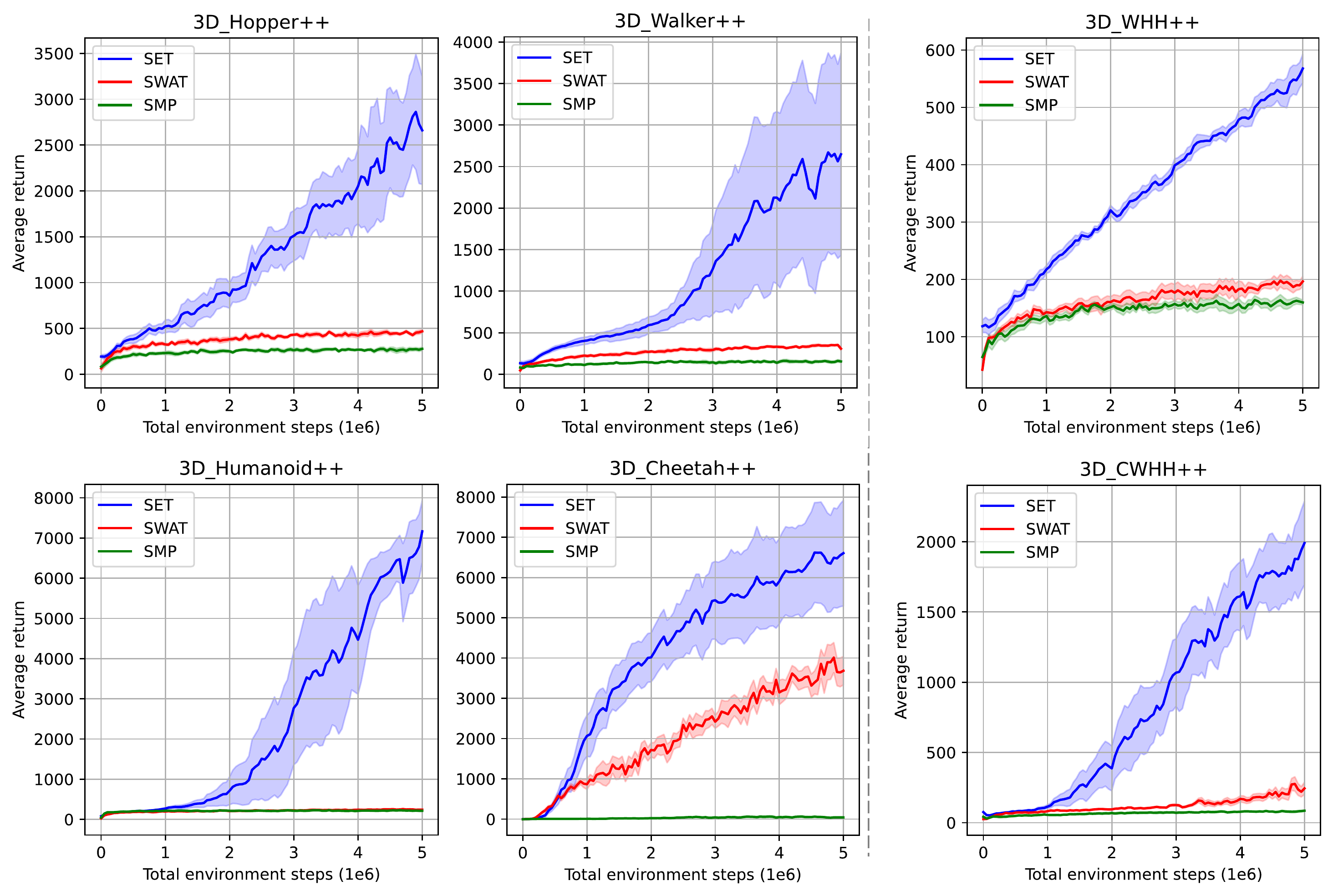}
\vspace{-.2in}
\caption{Multi-task performance of our method \name{} compared to the morphology-agnostic RL baselines: SWAT and SMP. Training curves on 6 collections of environments. 
The shaded area represents the standard error.}
\label{fig:multi-task}
\vspace{-.15in}
\end{figure*}

\subsection{Actor and Critic}\label{sec3.4}
With multiple layers of message fusion on the morphological graph of the agent, we are ready to output the actor policy $\pi_\theta$ and critic Q-function $Q_{\pi_\theta}$ to obtain the training objective of morphology-agnostic RL. Notably, the action in 3D-SGRL setting has been extended to be the three values of the torques projected onto the three rotation axes of each joint, driven by the actuators attached. This is attained by firstly reading out the subequivariant vector from the output of the $L$-th layer of our~\name{}, namely,
\begin{align}
\label{eq:readout}
    \vec{\mT}_i=[\vec{\vu}^{(L)}_i, \vec{\vg},\vec\vd]\sigma_{\mM}\left(\mM^{(L)}_i\right)\mW_{\vec{\mT}},
\end{align}
where $\vec{\vu}^{(L)}_i=\vec{\mZ}^{(L)}_i \mW^{(L)}_{\vec{\vu}}$ is a mixing of channels,  $[\vec{\vu}^{(L)}_i, \vec{\vg},\vec\vd]\in\sR^{3\times m'}$ is a stack of $\vec{\vu}^{(L)}_i$, $\vec\vg$ and $\vec\vd$ along the last dimension,  $\sigma_{\mM}$ is, again, an MLP: $\sR^{m\times m}\mapsto\sR^{m'\times m'}$, and $\mW_{\vec{\mT}} \in \sR^{m'\times 1}$ is a linear transformation.
Thanks to the $\text{O}_{\vec\vg}(3)$-equivariance of \name{} and the readout in \Cref{eq:readout}, the torque matrix $\vec{\mT}_i\in\sR^{3\times 1}$ is also $\text{O}_{\vec\vg}(3)$-equivariant. The scalars of the torques projected on three rotation axes of the joint are then naturally given by taking the inner products:
\begin{align}
\label{eq:action}
    \va_i\in\sR^{3}=[\vec{\mT}_i\cdot\vec{\vx}_i,\vec{\mT}_i\cdot\vec{\vy}_i,\vec{\mT}_i\cdot\vec{\vz}_i],
\end{align}
where $\va_i$ is the $\text{O}_{\vec\vg}(3)$-invariant output action of the actuators assigned to limb $i$.
By putting together all actions $\va_i$, $i\in\{1,\cdots,|\gV|\}$, the final output action $\va$ in \Cref{eq:graphrl} is collected.

The $\text{O}_{\vec\vg}(3)$-invariant Q-function $Q_{\pi_\theta}$ is similarly obtained by directly making use of the invariant $\mM^{(L)}_i$, given by,
\begin{align}
    Q_{\pi_\theta} = \mW_{Q_{\pi_\theta}}\text{vec}(\mM^{(L)}_i) + b_{Q_{\pi_\theta}},
\end{align}
where $\mW_{Q_{\pi_\theta}} \in \sR^{1 \times mm}, b_{Q_{\pi_\theta}} \in \sR$ collects the learnable weights and bias. Note that for learning actor policy $\pi_\theta$ and critic $Q_{\pi_\theta}$, we employ two separate~\name{}s, since for computing $Q_{\pi_\theta}$ we need to additionally concatenate the action $\va_i$ into the input of the first layer, \ie, $\vh^{(0)}_i=[\vh_i, \va_i]$. Here, we concatenate $\va_i$ to $\vh_i$ rather than $\mZ_i^{(0)}$ owing to the  $\text{O}_{\vec\vg}(3)$-invariance of $\va_i$. Formal proof of the equivariance of~\name{} and the invariance of the output action and critic are presented in \Cref{sec:proof}.

\section{Benchmark Construction}
In this section, we introduce technical details in constructing our challenging benchmarks in 3D-SGRL. 

\paragraph{Environments and Agents}
\label{sec:env-agent}
The environments in our 3D-SGRL are modified from the default 2D-planar setups in MuJoCo~\cite{todorov2012mujoco}.
Specifically, we extend agents in environments including \texttt{Hopper}, \texttt{Walker}, \texttt{Humanoid} and \texttt{Cheetah}~\cite{huang2020one} into 3D counterparts. 
For the multi-task training, we additionally construct several variants of each of these agents, as displayed in \Cref{tab:environments}. We create the following collections of environments with these variants, and categorize the collections into two settings: \emph{in-domain} and \emph{cross-domain}. 
For in-domain, there are four collections: (1) three variants of \texttt{3D Hopper} [\texttt{3D\_Hopper++}], (2) eight variants of \texttt{3D Walker} [\texttt{3D\_Walker++}], (3) eight variants of \texttt{3D Humanoid} [\texttt{3D\_Humanoid++}], (4) ten variants of \texttt{3D Cheetah} [\texttt{3D\_Cheetah++}].
The cross-domain environments are combinations of in-domain environments: (1) Union of \texttt{3D\_Walker++}, \texttt{3D\_Humanoid++} and \texttt{3D\_Hopper++} [\texttt{3D\_WHH++}], (2) Union of \texttt{3D\_Cheetah++}, \texttt{3D\_Walker++}, \texttt{3D\_Humanoid++} and \texttt{3D\_Hopper++} [\texttt{3D\_CWHH++}].
We keep 20\% of the variants as the zero-shot testing set and use the rest for training.   
In particular, the standard half-cheetah~\cite{wawrzynski2007learning,wawrzynski2009cat} has been so far designed as a 2D-Planar model with the morphology of a walking animal. However, in 3D-SGRL, the half-cheetah is highly vulnerable to falling over in its locomotion, adding more difficulties to policy optimization.
On account of this limitation, we extend the model to a full-cheetah with one torso, four legs, and one tail made of 14 limbs, enabling it stronger locomotion ability to explore in our 3D-SGRL environments. 
More design details are shown in \Cref{sec:env-details}.

\paragraph{State Space} We take the initial position of the agent's torso as the center, and randomly select its initial orientation and the destination within a radius of $R$. When the agent reaches the assigned target position, we set another destination for it. To relieve the agent from falling down when turning at a high speed, we set the radius $R=10 km$ by default so that the agent will turn less frequently in an episode. We also set $R\in[10m,20m]$ as ``v2-variants'', which is more difficult since the agent will change the direction more frequently.

\paragraph{Action Space} The action space is enlarged by changing the type of the joint of torso from ``slide-slide-hinge'' to ``free'' and adding two more actuators that rotate around different axes of the joint. This allows the agent to have full DoFs to turn and move in arbitrary directions starting from arbitrary initial spatial configurations.

\begin{figure}[t!]
\centering
\includegraphics[width=\linewidth]{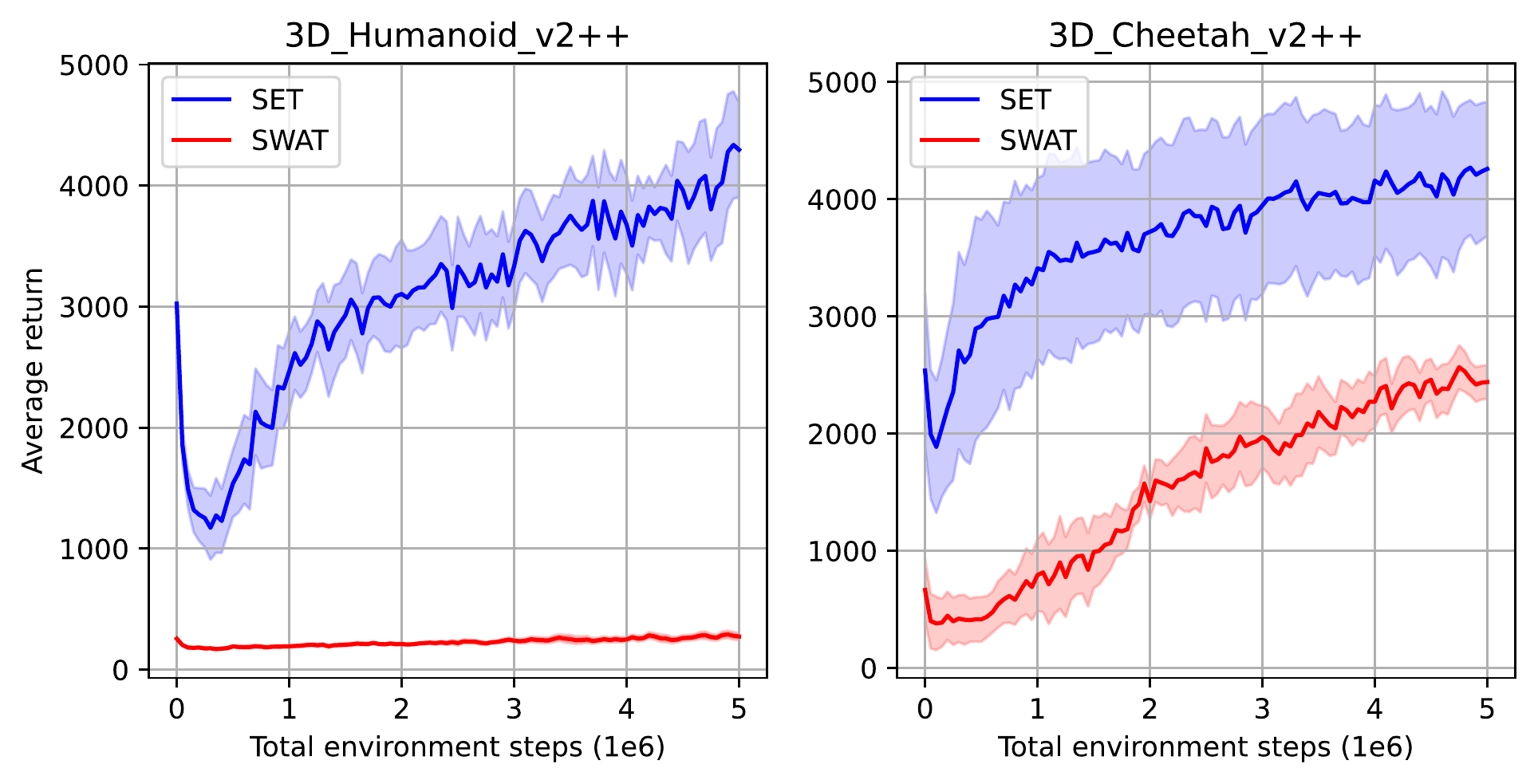}
\vspace{-.32in}
\caption{Training curves of v2-variants on \texttt{3D\_Humanoid++} and \texttt{3D\_Cheetah++}. 
}
\label{fig:v2-task}
\vspace{-.13in}
\end{figure}

\paragraph{Termination and Reward}
The goal in 3D-SGRL environments is learning to turn and move towards the assigned destination as fast as possible without falling over. 
Episode Termination follows that of the morphology-agnostic RL benchmark, but we modify the cheetah's termination to be the time it falls over or squats still.
The reward consists of four parts.
\textbf{1.} Alive bonus: Every timestep the agent is alive, it gets a reward of a fixed value 1 (\texttt{3D Cheetah}'s is 0 due to the stability of its morphological structure); 
\textbf{2.} Locomotion reward: It is a reward for moving towards the assigned target which is measured as  $\texttt{(distance\_before\_action}$  $\texttt{-distance\_after\_action)/dt}$, where $\texttt{dt}$ is the time between consecutive actions. This reward will be positive if the agent is close to the target position; 
\textbf{3.} Control cost: It is a cost for penalizing the agent if it takes actions that are too large. It is measured as $0.001 * \sum_{k=1}^{K}(\va_k)^2$; 
\textbf{4.} Forward reward (not available for \texttt{3D Hopper}): It is a reward of moving forward measured as  $\texttt{(coordinate\_after\_action -}$ $\texttt{ coordinate\_before\_action)}$$\cdot$$\texttt{forward\_direct}$ $\texttt{ion\_of\_torso/dt}$. This reward will be positive if the agent moves in the forward direction of torso.

\section{Evaluations and Ablations}

This section first introduces the baselines and implementations, then compares the performance of different methods on our 3D benchmarks and reports the ablation studies for the design of our method. 

\subsection{Baseline, Metric and Implementation}

\paragraph{Baselines} We compare our method \name{} against state-of-the-art methods \SMP{}~\citep{huang2020one} and \SWAT{}~\citep{hong2021structure}. 
We also compare \name{} with standard TD3-based non-morphology-agnostic RL: \Monolithic{} in single-tasks. 
Please refer to \Cref{sec:repro} for more details about baselines.

\paragraph{Metrics} \textbf{1.} Multi-task with different morphologies: For each multi-task environment discussed in \Cref{sec:env-agent}, a single policy is simultaneously trained on multiple variants. The policy in each plot is trained jointly on the training set (80\% of variants from that environment) and evaluated on these seen variants.
\textbf{2.} Zero-Shot Generalization:
We take the trained policies from multi-task and test on the unseen zero-shot testing variants.
\textbf{3.} Evaluation on v2-variants:
We evaluate \name{} in a transfer learning setting where the trained policies from multi-task are tested and fine-tuned on the v2-variants environments.
\textbf{4.} Single-task Learning: 
The policy in each plot is trained on one morphology variant and evaluated on this variant.

\paragraph{Implementations} We adopt the same input information and TD3~\citep{fujimoto2018addressing} as the underlying reinforcement learning algorithm for training the policy over all baselines, ablations, and \name{} for fairness. We implement \name{} in the \SWAT{} codebase.  There is no weight sharing between actor $\pi_\theta$ and critic $Q_{\pi_\theta}$. 
Each experiment is run with three seeds to report the mean and the standard error. The reward for each environment is calculated as the sum of instant rewards across an episode. The value of the maximum timesteps of an episode is $1,000$.

\begin{table}[t!]
    \centering
    \caption{Comparison in zero-shot evaluation on the test set. Note that we omit the lacking part in the name of morphologies.}
    \label{tab:zero-shot}
    \resizebox{\linewidth}{!}{
    \begin{tabular}{llll}
    \toprule
    \textbf{Environment}&\textbf{\name{}}&\textbf{\SWAT{}}&\textbf{\SMP{}}\\
     \midrule
    \multicolumn{4}{l}{in-domain (\texttt{3D\_Walker++}, \texttt{3D\_Humanoid++}, \texttt{3D\_Cheetah++})} \\
    \midrule
    \texttt{3d\_walker\_3}&$\textbf{276.2}\pm17.4$&$207.0\pm52.7 $&$56.8\pm15.1$\\
     \texttt{3d\_walker\_6}&$\textbf{431.3}\pm146.2$&$358.0\pm58.9$&$143.4\pm50.7$\\    
     \midrule
     \texttt{3d\_humanoid\_7}&$\textbf{244.8}\pm7.9$&$170.3\pm51.7$&$190.9\pm16.2$\\
     \texttt{3d\_humanoid\_8}&$\textbf{299.6}\pm23.7$&$141.4\pm22.1$&$185.4\pm9.2$\\
     \midrule     \texttt{3d\_cheetah\_11}&$\textbf{4643.9}\pm292.6$&$1785.3\pm999.3$&$2.0\pm2.9$\\
    \texttt{3d\_cheetah\_12}&$\textbf{916.0}\pm39.7$&$744.1\pm317.1$&$29.8\pm10.7$\\

         \midrule
         
         \multicolumn{4}{l}{cross-domain (\texttt{3D\_CWHH++})}\\
    \midrule
    \texttt{3d\_walker\_3}&$\textbf{206.8}\pm37.4$&$17.9\pm13.7$&$18.0\pm22.9$\\
     \texttt{3d\_walker\_6}&$\textbf{243.7}\pm32.3$&$114.9\pm40.3$&$103.9\pm1.8$\\    
     \midrule
     \texttt{3d\_humanoid\_7}&$\textbf{161.9}\pm3.4$&$152.0\pm6.8$&$124.2\pm15.7$\\
     \texttt{3d\_humanoid\_8}&$\textbf{180.0}\pm6.5$&$156.6\pm1.7$&$129.3\pm0.1$\\   
     \midrule     \texttt{3d\_cheetah\_11}&$\textbf{1078.1}\pm722.8$&$4.3\pm1.6$&$6.2\pm0.5$\\
    \texttt{3d\_cheetah\_12}&$\textbf{3038.3}\pm2803.3$&$349.7\pm304.3$&$6.6\pm1.2$\\
     
    \bottomrule
    
    \end{tabular}
    }
    \vspace{-.1in}
\end{table}

\begin{figure}[b!]
\centering
\includegraphics[width=\linewidth]{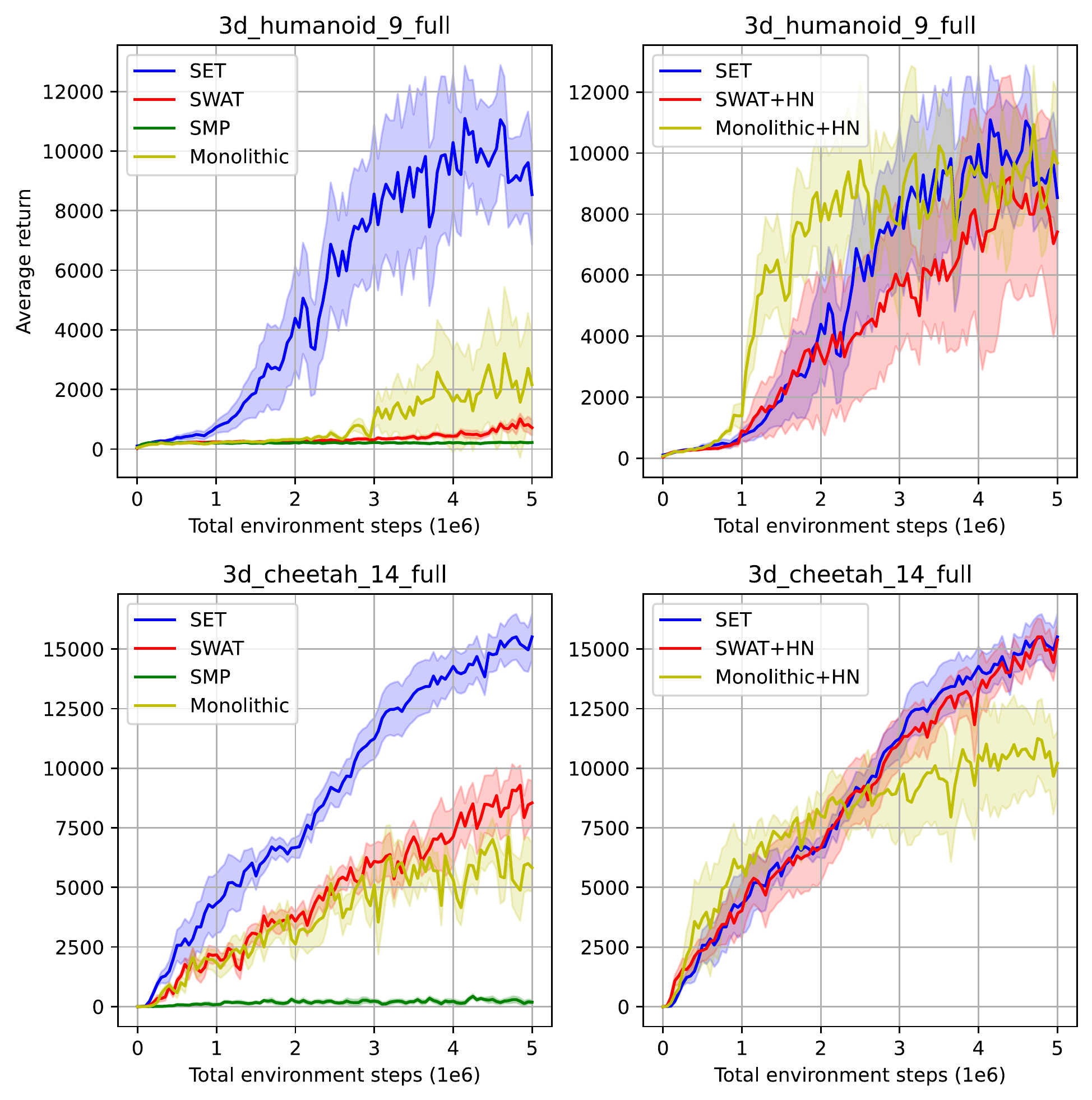}
\vspace{-.32in}
\caption{Training curves of single-task on
\texttt{3d\_humanoid} \texttt{\_9\_full} and \texttt{3d\_cheetah\_14\_full}.  On the left-hand side, we present the comparison with baselines, while on the right-hand side, we present the comparison with invariant methods.
% The shaded area represents the standard error.
}
\label{fig:single-task}
\vspace{-.1in}
\end{figure}

\subsection{Main Results}

\paragraph{Multi-task with different morphologies}
\label{sec:benchmark}
As shown in \Cref{fig:multi-task}, our \name{} outperforms all baselines by a large margin in all cases, indicating the remarkable superiority of taking into account the subequivariance upon Transformer. The baselines fail to achieve meaningful returns in most cases, which is possibly due to the large exploration space in our 3D-SGRL environments and they are prone to get trapped in local extreme points.

\paragraph{Zero-Shot Generalization}
\label{sec:zeroshot}
During test time, we assess the trained policy on a set of held-out agent morphologies. \Cref{tab:zero-shot} records the results of both in-domain and cross-domain settings. 
The training and zero-shot testing variants are listed on \Cref{tab:environments}.
For example, \name{} is trained on \texttt{3D\_Humanoid++} without  \texttt{3d\_humanoid\_7\_left\_leg} and \texttt{3d\_humanoid\_8\_right\_knee}, while these two excluded environments are used for testing. \Cref{tab:zero-shot} reports the average performance and the standard error over 3 seeds, where the return of each seed is calculated over 100 rollouts. Once again, we observe that \name{} yields better performance.
% than SWAT and SMP. 

\paragraph{Evaluation on v2-variants}
The v2-variants ($R=10\sim20 m$) are more challenging. We conduct two-stage training in this scenario. In the first stage, we train the policy under the multi-task setting where $R=10km$. The results and related demos are in \Cref{sec:v2}.
In the second stage, we transfer the currently-trained policy to the $R=10\sim20 m$ setting on \texttt{3D\_Cheetah++} and \texttt{3D\_Humanoid++}. 
It is seen from \Cref{fig:v2-task} that \name{} is able to further improve the performance upon the first stage, while \SWAT{} hardly receives meaningful performance gain especially on \texttt{3D\_Humanoid++}.

\paragraph{Single-task Learning} 
Apart from \SMP{} and \SWAT{}, we implement another baseline \Monolithic{} for reference. \Cref{fig:single-task} displays the performance on \texttt{3d\_humanoid\_9\_full} and \texttt{3d\_cheetah\_14\_full}. In line with the observations in~\cite{dong2022low}, the GNN-based method \SMP{} is worse than the MLP-based model \Monolithic{}; but different from the results in~\cite{dong2022low}, \SWAT{} still surpasses \Monolithic{} on \texttt{3d\_cheetah\_14\_full}. We conjecture \SWAT{} benefits from the application of Transformer that is expressive enough to characterize the variation of our \texttt{3d\_cheetah\_14\_full} environments. Our model \name{} takes advantage of both the expressive power of the Transformer-akin model and the rational constraint by subequivariance, hence it delivers much better performance than all other methods.

\begin{figure}[t!]
\centering
\includegraphics[width=0.5\linewidth]{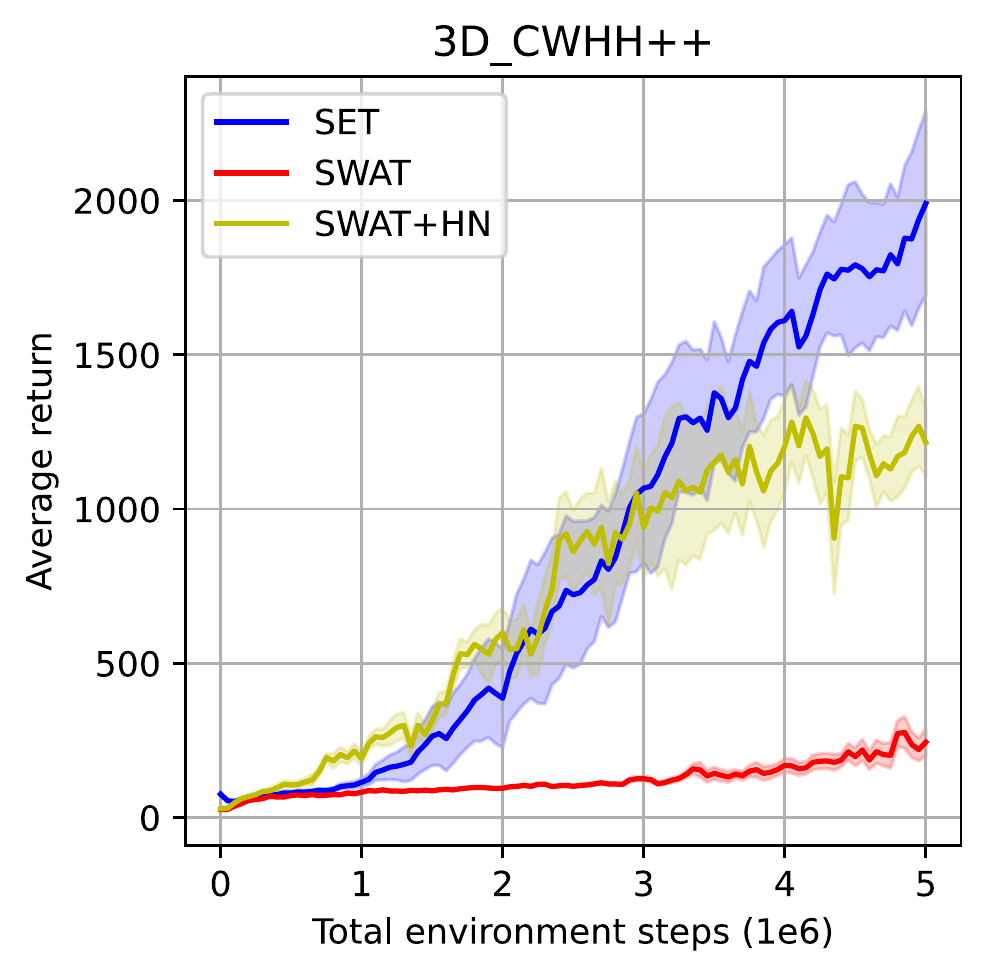}
\vspace{-.2in}
\caption{Training curves of multi-task on \texttt{3D\_CWHH++}. The comparison with invariant methods.
% The shaded area represents the standard error.
}
\label{fig:cwhh_hn}
\vspace{-.1in}
\end{figure}

\subsection{Comparison with Invariant Methods}
\label{sec:compare}
Invariant methods have been widely utilized in the 3D RL literature. For instance, in humanoid control, the presence of gravity allows for the normalization of state and action spaces in the heading (yaw) direction (\eg, a recent work~\cite{won2022physics}). This heading normalization (HN) technique transforms the global coordinate frame into a local coordinate frame, enabling the input geometric information to be mapped to a rotation- and translation-invariant representation. We compare \name{} with the following invariant variants: \textbf{1.} \SWAT{}+HN: a state-of-the-art morphology-agnostic baseline that uses the heading normalization, and \textbf{2.} Monolithic+HN: a standard TD3-based non-morphology-agnostic baseline that uses the heading normalization. As shown in \Cref{fig:single-task} and \Cref{fig:cwhh_hn}, \name{} can only be considered on par with \SWAT{}+HN, since heading normalization can achieve heading-equivariance by construction.

Indeed, there is a limitation of heading normalization in that it assumes a consistent definition of the ``forward" direction across all agents. Without a consistent ``forward" direction, the normalization scheme would need to be redefined for each individual agent, which could limit its transfer ability to different types of agents or environments. On the contrary, equivariant methods, such as the one proposed in our work, can be more generalizable as they do not rely on a specific normalization scheme and can adapt to different transformations in the environment.
We design a simple experiment to verify the above statement by translating the ``forward" direction of the agent via a certain bias angle during testing. \Cref{tab:bias} demonstrates the significant performance degradation caused by adding bias in the heading normalization.
Moreover, we can support this point through zero-shot generalization experiments, where we evaluate the trained policies from multi-task on unseen zero-shot testing variants. \Cref{tab:hn-zero-shot} demonstrates that \name{} has stronger generalization ability compared to \SWAT{}+HN.
For more detailed discussions, please refer to \Cref{sec:discuss}.

\begin{table}[t!]
    \centering
    \caption{Single-task performance with added bias in the heading normalization. The table header (the first row of the table) represents the environment and the bias. }
    \label{tab:bias}
    \resizebox{\linewidth}{!}{
    \begin{tabular}{ccccc}
    \toprule
     \multirow{2}{*}{Methods}&\multicolumn{2}{c}{\texttt{3d\_humanoid\_9\_full}}&\multicolumn{2}{c}{\texttt{3d\_cheetah\_14\_full}}\\
     & $0^\circ$ & $180^\circ$  & $0^\circ$ & $180^\circ$ \\
     \midrule
     \Monolithic{}+HN & $13142.2\pm2840.2$	&$57.8\pm12.0$	&$11357.4\pm1933.0$	&$-3.2\pm0.7$\\
     \SWAT{}+HN & $8517.7\pm1796.4$	&$92.3\pm17.8$	&$15924.9\pm543.1$	&$-1.2\pm0.4$\\
     \name{} & $9931.9\pm632.0$	&$10106.4\pm2023.4$	&$14987.9\pm710.7$	&$14957.9\pm758.0$\\
    \bottomrule
    \end{tabular}}
\vspace{-.1in}
\end{table}

\begin{table}[t!]
    \centering
    \caption{Compared with Heading Normalization in zero-shot evaluation on the test set. Note that we omit the lacking part in the name of morphologies.}
    \label{tab:hn-zero-shot}
    \resizebox{0.8\linewidth}{!}{
    \begin{tabular}{lll}
    \toprule
    \textbf{Environment}&\textbf{\name{}}&\textbf{\SWAT{}+HN}\\
         \midrule
         \multicolumn{3}{l}{cross-domain (\texttt{3D\_CWHH++})}\\
    \midrule
    \texttt{3d\_walker\_3}&$\textbf{206.8}\pm37.4$&$26.3\pm72.4$\\
     \texttt{3d\_walker\_6}&$\textbf{243.7}\pm32.3$&$156.8\pm11.1$\\    
     \midrule
     \texttt{3d\_humanoid\_7}&$\textbf{161.9}\pm3.4$&$130.2\pm2.1$\\
     \texttt{3d\_humanoid\_8}&$\textbf{180.0}\pm6.5$&$152.9\pm36.8$\\   
     \midrule     \texttt{3d\_cheetah\_11}&$\textbf{1078.1}\pm722.8$&$786.5\pm779.3$\\
    \texttt{3d\_cheetah\_12}&$\textbf{3038.3}\pm2803.3$&$2517.3\pm2113.9$\\
    \bottomrule
    \end{tabular}}
\vspace{-.1in}
\end{table}

\subsection{Ablation}
\label{sec:ablation}

\begin{figure}[t!]
\centering
\includegraphics[width=\linewidth]{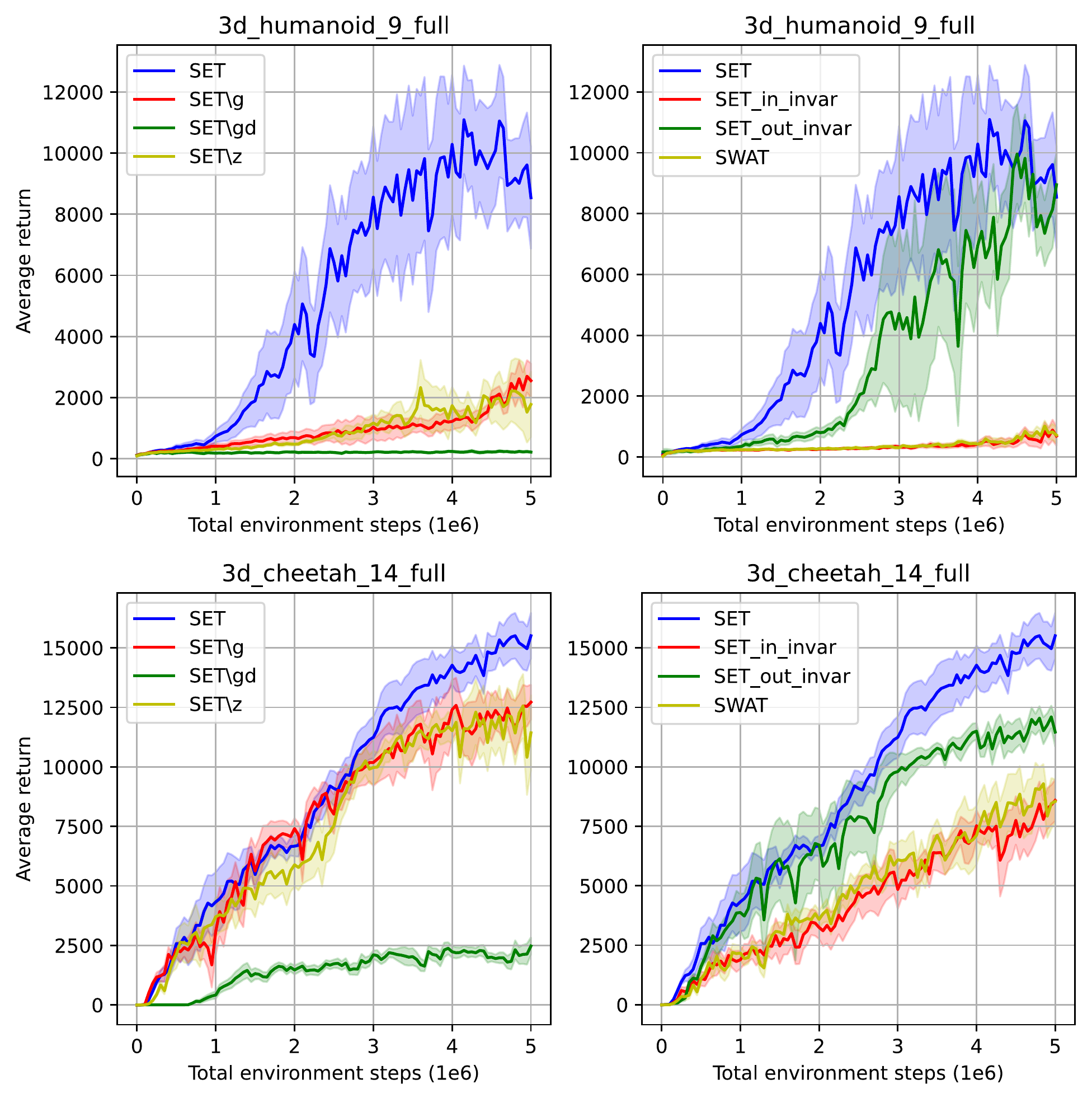}
\vspace{-.33in}
\caption{Training curves of ablations of \name{} on
\texttt{3d\_humanoid} \texttt{\_9\_full} and \texttt{3d\_cheetah\_14\_full}. 
% The shaded area represents the standard error.
}
\label{fig:ablation}
\vspace{-.1in}
\end{figure}

\begin{figure}[t!]
\centering
\includegraphics[width=0.9\linewidth]{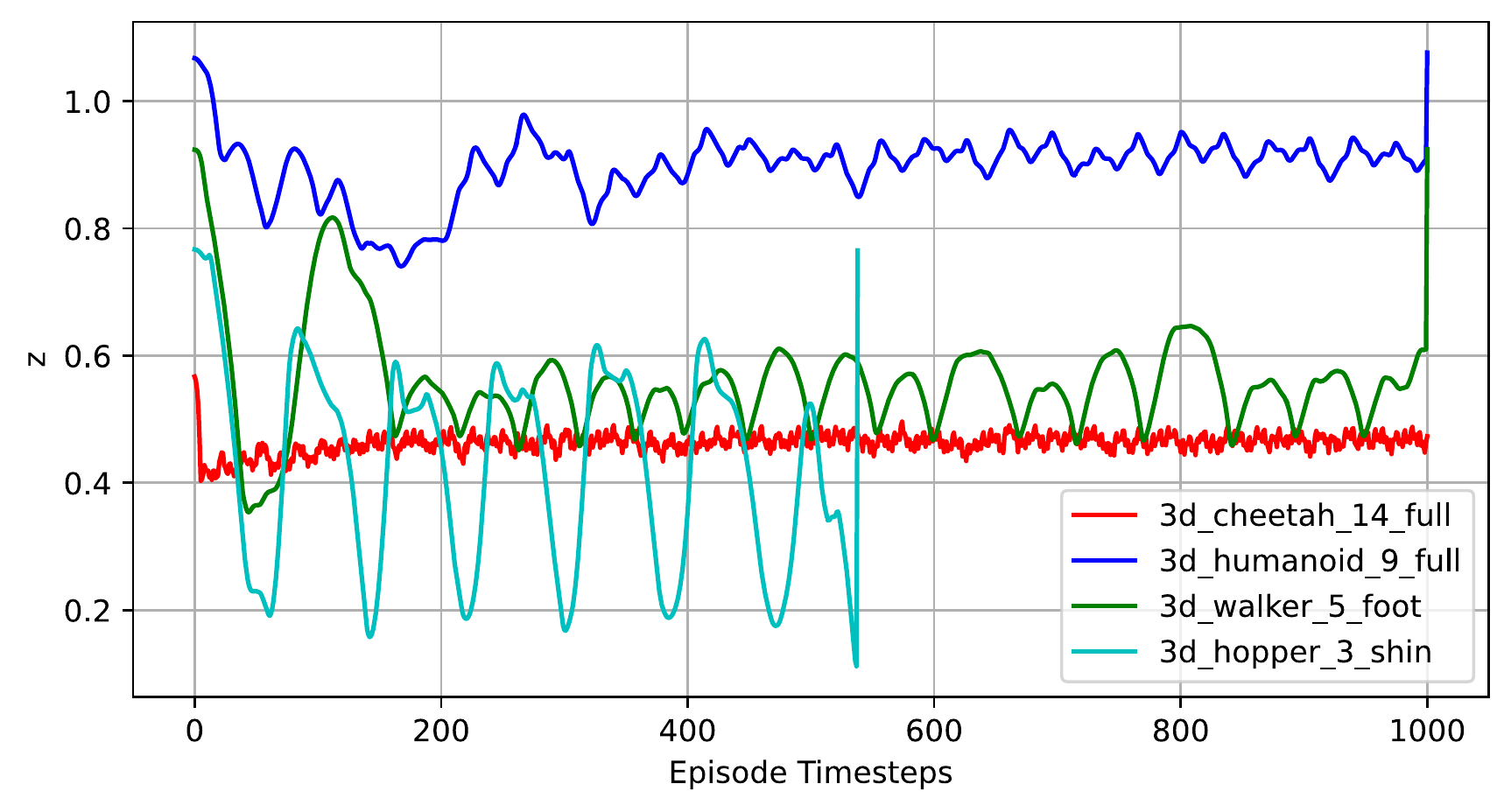}
\vspace{-.23in}
\caption{Average height of all limbs.
}
\label{fig:z}
\vskip -0.1in
\end{figure}

We ablate the following variants in \Cref{fig:ablation}: 
\textbf{1.} \name{}$\backslash$g: an $\text{O}(3)$-equivariant model, where gravity $\vec{\vg}$ is removed from the external force and concatenated into the scalar input, $\vh^{(0)}_i=[\vh^{(0)}_i, \vec{\vg}]$;
\textbf{2.} \name{}$\backslash$gd: an $\text{O}(3)$-equivariant model, where both $\vec{\vg}$ and $\vec{\vd}$ are considered as scalars: $\vh^{(0)}_i=[\vh^{(0)}_i, \vec{\vg}, \vec{\vd}]$;
\textbf{3.} \name{}$\backslash$z: an $\text{O}_{\vec\vg}(3)$-equivariant model without \Cref{eq:z}, by omitting the height $\Vec{\vp}_i^{z}$;
\textbf{4.} \name{}\_in\_invar: a non-equivariant model without all geometric vectors, instead taking them as the scalar input, $\vh^{(0)}_i=[\vh^{(0)}_i, \vec{\mZ}_{i}, \vec{\vg}, \vec{\vd}]$; 
\textbf{5.} \name{}\_out\_invar: an $\text{O}_{\vec\vg}(3)$-equivariant model by replacing the action output by the projection strategy in \Cref{eq:action} with an $\text{O}_{\vec\vg}(3)$-invariant mapping $\va_i = \mW_{\pi_\theta}\text{vec}(\mM^{(L)}_i) + b_{\pi_\theta}$.

\textbf{1.} \name{}$\backslash$g and \name{}$\backslash$z, compared with \name{}, gain close performance on \texttt{3d\_cheetah\_14\_full}, but are much worse on \texttt{3d\_humanoid\_9\_full}. This is reasonable, as the agent \texttt{3d\_cheetah\_14\_full} has four legs and can locomote stably (see \Cref{fig:z}). It is thus NOT so essential to consider the effect of gravity and the height to the ground on \texttt{3d\_cheetah\_14\_full}. As for \texttt{3d\_humanoid\_9\_full} with 2 legs, however, it is important to sense the direction of gravity and detect the height to avoid potential falling down, hence the correct modeling of gravity and the height are necessary for locomotion policy learning.      
\textbf{2.} The performance of \name{}$\backslash$gd  is poor in both cases, indicating that maintaining the direction information of the task guidance is indispensable.
\textbf{3.} \name{}\_in\_invar behaves much worse than \name{}, which verifies the importance to incorporate subequivariance into our model design. 
\textbf{4.} \name{}\_out\_invar is worse than \name{} but already exceeds other variants. The equivariant output $\vec{\mT}_i$ in \name{} contains rich orientation information, and it is more direct to obtain the output torque by projecting  $\vec{\mT}_i$, than \name{}\_out\_invar which uses the invariant matrix $\mM^{(L)}_i$ to predict the action.

\section{Discussion}

In current machine learning research, equivariance and attention are both powerful ideas. To learn a shared graph-based policy in 3D-SGRL, we design \name{}, a novel transformer model that preserves geometric symmetry by construction. Experimental results strongly support the necessity of encoding symmetry into the policy network, which demonstrates its wide applicability in various 3D environments.
We also compare the Monolithic MLP-based model using heading normalization for single-task training in \Cref{fig:single-task}. It can be found that a simple MLP with heading normalization may outperform the benefits brought by equivariance and attention. Therefore, in comparison to traditional methods in single-task settings, we cannot guarantee that all humanoids and legged robots will experience considerable enhancement when using our equivariant methods. 
In this work, our main contribution is extending the 2D benchmark to 3D for morphology-agnostic RL, which mainly addresses challenges in multi-task learning with agents of inhomogeneous morphology where MLP may not be applicable.
Although these are just initial steps, we believe that further exploration of this research direction will lead to valuable contributions to the research community.

\section*{Acknowledgements}
This work is jointly funded by 
``New Generation Artificial Intelligence" Key Field Research and Development Plan of Guangdong Province (2021B0101410002),
the National Science and Technology Major Project of the Ministry of Science and Technology of China (No.2018AAA0102900),
the Sino-German Collaborative Research Project Crossmodal Learning (NSFC 62061136001/DFG TRR169),
THU-Bosch JCML Center,
the National Natural Science Foundation of China under Grant U22A2057, 
the National Natural Science Foundation of China (No.62006137), 
Beijing Outstanding Young Scientist Program (No.BJJWZYJH012019100020098), 
and 
Scientific Research Fund Project of Renmin University of China (Start-up Fund Project for New Teachers).
We sincerely thank the reviewers for their comments that significantly improved our paper's quality. 
Our heartfelt thanks go to Yu Luo, Tianying Ji, Chengliang Zhong, and Chao Yang for fruitful discussions. 
Finally, Runfa Chen expresses gratitude to his fiancée, Xia Zhong, for her unwavering love and support.

\bibliography{SGRL}
\bibliographystyle{icml2023}

%%%%%%%%%%%%%%%%%%%%%%%%%%%%%%%%%%%%%%%%%%%%%%%%%%%%%%%%%%%%%%%%%%%%%%%%%%%%%%%
%%%%%%%%%%%%%%%%%%%%%%%%%%%%%%%%%%%%%%%%%%%%%%%%%%%%%%%%%%%%%%%%%%%%%%%%%%%%%%%
% APPENDIX
%%%%%%%%%%%%%%%%%%%%%%%%%%%%%%%%%%%%%%%%%%%%%%%%%%%%%%%%%%%%%%%%%%%%%%%%%%%%%%%
%%%%%%%%%%%%%%%%%%%%%%%%%%%%%%%%%%%%%%%%%%%%%%%%%%%%%%%%%%%%%%%%%%%%%%%%%%%%%%%
\newpage
\appendix
\onecolumn
% \section{Appendix}

% You can have as much text here as you want. The main body must be at most $8$ pages long.
% For the final version, one more page can be added.
% If you want, you can use an appendix like this one, even using the one-column format.

\section{Proofs}
\label{sec:proof}
In this section, we theoretically prove that our proposed SubEquivariant Transformer (\name{}), and the final output action and critic Q-function value preserve the symmetry as desired. We start by verifying our design in~\name{}.

\begin{theorem}
\label{prop:set-equ}
Let $(\vec\mZ', \vh')=\varphi(\vec\mZ, \vec\vg, \vec\vd, \vh)$, where $\varphi$ is one layer of our~\name{} specified from \Cref{eq:M} to \Cref{eq:h}. Let $(\vec\mZ'^\ast, \vh'^\ast)=\varphi(\mO\vec\mZ, \vec\vg, \mO\vec\vd, \vh), \forall \mO\in\text{O}_{\vec\vg}(3)$. Then, we have $(\vec\mZ'^\ast, \vh'^\ast)=(\mO\vec\mZ', \vh')$, indicating $\varphi$ is $\text{O}_{\vec\vg}(3)$-equivariant.
\end{theorem}

\begin{proof}
In the first place, we have $\vec\vm_i^\ast = \mO\vec\mZ_i\mW_{\vec\vm}=\mO\vec{\vm}_i$. For the message $\mM_i$, we have,
\begin{align}
\mM^\ast_i&=\sigma_\mM\left(\sigma_{\vec\vm}\left([\vec{\vm}^{\ast}_i,\vec\vg,\mO\vec{\vd}]^\top[\vec{\vm}^{\ast}_i,\vec\vg,\mO\vec{\vd}]\right),\vh_i\right),\\
&=\sigma_\mM\left(\sigma_{\vec\vm}\left([\mO\vec{\vm}_i,\vec\vg,\mO\vec{\vd}]^\top[\mO\vec{\vm}_i,\vec\vg,\mO\vec{\vd}]\right),\vh_i\right),\\
\label{eq:111}
&=\sigma_\mM\left(\sigma_{\vec\vm}\left(\begin{bmatrix} \vec\vm_i^\top\mO^\top\mO\vec\vm_i& \vec\vm_i^\top\mO^\top\vec\vg & \vec\vm_i^\top\mO^\top\mO\vec\vd\\
\vec\vg^\top\mO\vec\vm_i & \vec\vg^\top\vec\vg & \vec\vg^\top\mO\vec\vd \\
 \vec\vd^\top\mO^\top\mO\vec\vm_i & \vec\vd^\top\mO^\top\vec\vg & \vec\vd^\top\mO^\top\mO\vec\vd\end{bmatrix}\right),\vh_i\right), \\
 \label{eq:222}
 &=\sigma_\mM\left(\sigma_{\vec\vm}\left(\begin{bmatrix} \vec\vm_i^\top\vec\vm_i& \vec\vm_i^\top\vec\vg & \vec\vm_i^\top\vec\vd\\
\vec\vg^\top\vec\vm_i & \vec\vg^\top\vec\vg & \vec\vg^\top\vec\vd \\
 \vec\vd^\top\vec\vm_i & \vec\vd^\top\vec\vg & \vec\vd^\top\vec\vd\end{bmatrix}\right),\vh_i\right), \\
 \label{eq:333}
&=\sigma_\mM\left(\sigma_{\vec\vm}\left([\vec{\vm}_i,\vec\vg,\vec{\vd}]^\top[\vec{\vm}_i,\vec\vg,\vec{\vd}]\right),\vh_i\right) = \mM_i.
\end{align} 
From \Cref{eq:111} to \Cref{eq:222} we use the fact $\mO^\top\mO=\mI$ and $\mO^\top\vec\vg=\vec\vg$, by the definition of the group $\text{O}_{\vec\vg}(3)$. With the $\text{O}_{\vec\vg}(3)$-invariant message $\mM_i$, it is then immediately illustrated that the query $\vq_i$, key $\vk_i$, value message $\vv_j$, and the attention coefficient $\alpha_{ij}$ are all $\text{O}_{\vec\vg}(3)$-invariant, and value message $\vec{\vu}_j^\ast=\vec{\mZ}_j^\ast\mW_{\vec\vu}=\mO\vec{\mZ}_j\mW_{\vec\vu}=\mO\vec{\vu}_j$ is $\text{O}_{\vec\vg}(3)$-equivariant. Finally, we have,
\begin{align}
\vec{\mZ}'^{\ast}_i&=\mO\vec{\mZ}_i+\sum_j\left(\alpha_{ij}[\mO\vec{\vu}_j,\vec{\vg},\mO\vec{\vd}]\right)\mW_{\vec{\mZ}},\\
&=\mO\vec{\mZ}_i+\sum_j\left(\alpha_{ij}\mO[\vec{\vu}_j,\vec{\vg},\vec{\vd}]\right)\mW_{\vec{\mZ}},\\
&=\mO\left(\vec{\mZ}_i+\sum_j\left(\alpha_{ij}[\vec{\vu}_j,\vec{\vg},\vec{\vd}]\right)\mW_{\vec{\mZ}}\right),\\
&=\mO\vec{\mZ}',
\end{align}
and similarly,
\begin{align}
\vh'^\ast_i&=\text{LN}\left(\vh_i+\mW_\vh\sum_j\left(\alpha_{ij}\vv_j\right)  + \vb_{\vh}\right)=\vh'_i.
\end{align}
By going through all nodes $i\in\{1,\cdots,|\gV|\}$ the proof is completed.
\end{proof}
By iteratively applying \Cref{prop:set-equ} for $l\in\{1,\cdots,L\}$ layers, we readily obtain the $\text{O}_{\vec\vg}(3)$-equivariance of the entire~\name{}.

As for the actor and critic, we additionally have the following corollary.
\begin{corollary}
\label{prop:act-equ}
Let $\va, Q_{\pi_\theta}$ be the output action and the critic of 3D-SGRL with $\vec\mZ, \vec\vg, \vec\vd, \vh$ as input. Let $\va^\ast,  Q_{\pi_\theta}^\ast$ be the action and critic with $\mO\vec\mZ, \vec\vg, \mO\vec\vd, \vh$ as input, $\mO\in\text{O}_{\vec\vg}(3)$. Then, $(\va^\ast, Q^\ast)=(\va, Q)$, indicating the output action and critic preserve $\text{O}_{\vec\vg}(3)$-invariance.
\end{corollary}

\begin{proof}
    By \Cref{prop:set-equ}, we have $\vec{\mZ}^{(L)\ast}_i=\mO\vec{\mZ}^{(L)}_i$, and $\mM_i^{(L)\ast}=\mM_i^{(L)}$. Therefore, $\vec{\vu}^{(L)\ast}_i=\vec{\mZ}^{(L)\ast}_i\mW_{\vec\vu}^{(L)}=\mO\vec{\mZ}^{(L)}_i\mW_{\vec\vu}^{(L)}=\mO\vec{\vu}^{(L)\ast}_i$. Hence,
    \begin{align}
    \label{eq:555}
         \vec{\mT}^\ast_i&=[\mO\vec{\vu}^{(L)}_i, \vec{\vg},\mO\vec\vd]\sigma_{\mM}\left(\mM^{(L)}_i\right)\mW_{\vec{\mT}},\\
          \label{eq:666}
         &=\mO\left([\vec{\vu}^{(L)}_i, \vec{\vg},\vec\vd]\sigma_{\mM}\left(\mM^{(L)}_i\right)\mW_{\vec{\mT}}\right),\\
         &=\mO\vec{\mT}_i,
    \end{align}
    where \Cref{eq:555} to \Cref{eq:666}, again, leverages the fact that $\vec\vg=\mO\vec\vg$, given the definition of $\text{O}_{\vec\vg}$.
    Finally,
    \begin{align}
            \va^\ast_i&=[\vec{\mT}_i\mO^\top\mO\vec{\vx}_i,\vec{\mT}_i\mO^\top\mO\vec{\vy}_i,\vec{\mT}_i\mO^\top\mO\vec{\vz}_i],\\
            &=[\vec{\mT}_i\cdot\vec{\vx}_i,\vec{\mT}_i\cdot\vec{\vy}_i,\vec{\mT}_i\cdot\vec{\vz}_i]=\va_i,
    \end{align}
    and meanwhile,
    \begin{align}
        Q_{\pi_\theta}^\ast = \mW_{Q_{\pi_\theta}}\text{vec}(\mM^{(L)}_i) + b_{Q_{\pi_\theta}}=Q_{\pi_\theta},
    \end{align}
    since concatenating the $\text{O}_{\vec\vg}(3)$-invariant $\va$ into the input $\vh$ does not affect the $\text{O}_{\vec\vg}(3)$-invariance of the message $\mM_i^{(L)}$.
    
\end{proof}

\section{Related Works}\label{sec:related_work}
\paragraph{Morphology-Agnostic RL}
In recent years, we have seen the emergence and development of multi-task RL with the inhomogeneous morphology setting, where the state and action spaces are different across tasks~\cite{devin2017learning,chen2018hardware,d2020sharing}.
The morphology-agnostic approach, which learns policies for each joint using multiple message passing schemes, decentralizes the control of multi-joint robots.
In order to deal with the inhomogeneous setting, NerveNet~\cite{wang2018nervenet}, DGN~\cite{pathak2019learning} and \SMP{}~\cite{huang2020one} represent the morphology of the agent as a graph and deploy GNNs as the policy network.
\Amorpheus{}~\cite{kurin2020cage}, \SWAT{}~\cite{hong2021structure} and \Solar{}~\cite{dong2022low}  utilize the self-attention mechanism instead of GNNs for direct communication. 
In morphology-agnostic RL, both of their investigations demonstrate that the graph-based policy has significant advantages over a monolithic policy.
Our work is based on \SWAT{} and introduces a set of new benchmarks that relax the over-simplified state and action space of existing works to a much more challenging scenario with immersive search space.

\paragraph{Geometrically Equivariant Models}
Prominently, there are certain symmetries in the physical world and there have been a number of studies about group equivariant models~\cite{cohen2016group, cohen2016steerable, worrall2017harmonic}.
In recent years, a field of research known as geometrically equivariant graph neural networks~\cite{han2022geometrically}, leverages symmetry as an inductive bias in learning. 
These models are designed such that their outputs will rotate/translate/reflect in the same way as the inputs, hence retaining the symmetry. 
Several methods are used to achieve this goal, such as using irreducible representation to solve group convolution~\cite{thomas2018tensor,fuchs2020se3} or utilizing invariant scalarization~\cite{villar2021scalars} like taking the inner product~\cite{satorras2021en,huang2022equivariant,han2022learning}. 
Along with GMN's~\cite{huang2022equivariant} and SGNN's~\cite{han2022learning} approaches to scalarization, our method is a member of this family. 
In a Markov decision process (MDP) with symmetries~\cite{van2020mdp}, there are symmetries in the state-action space where policies can thus be optimized in
the simpler abstract MDP. \citet{van2020mdp} attempts to learn equivariant policy and invariant value networks in 2D toy environments. Our work focuses on the realization of this motivation in more complex 3D physics simulation environments.

\section{More Experimental Details}
\subsection{Environments and Agents}
\label{sec:env-details}

We choose the following environments from morphology-agnostic RL benchmark~\cite{huang2020one} to evaluate our methods: \texttt{Hopper++}, \texttt{Walker++}, \texttt{Humanoid++}, \texttt{Cheetah++}. To facilitate the study of subequivariant graph reinforcement learning across these agents, we modify the 2D-Planar agents and extend them into 3D agents. Specifically, we modify the joint of torso from the combination of ``slide-slide-hinge'' type to ``free'' type. Normally, each joint of the agent in the 2D-Planar environment has only one hinge-type actuator to make it rotate around $y$-axis. In order to make the agent more flexible to explore and optimize the learning process, we expand its action space including increasing the number of hinge-type actuators from 1 to 3,  thus the DoF of each joint is also enlarged to 3. The two newly-added actuators enable the joint to basically rotate around $x$-axis and $z$-axis, respectively.

\texttt{3D Hopper}: The rotation range of the joint's two newly-added actuators is limited to $[-\frac{10}{180} \pi, \frac{10}{180} \pi]$.

\texttt{3D Walker}: The legs of \texttt{3D Walker} is designed with reference to the legs of standard \texttt{3D Humanoid}~\cite{tassa2012synthesis}.
The rotation range of each joint is limited to new intervals. The rotation range of the joints in left and right leg are the same, we only show the intervals of a joint of the left leg:
\begin{equation*}
\aligned
\text{the joint of thigh:}&[-\frac{25}{180} \pi,\frac{5}{180}],[-\frac{20}{180} \pi,\frac{110}{180} \pi],[-\frac{60}{180} \pi,\frac{35}{180} \pi],\\
\text{the joint of shin:}&[-\frac{1}{180} \pi,\frac{1}{180}],[-\frac{160}{180} \pi,-\frac{2}{180} \pi],[-\frac{1}{180} \pi,\frac{1}{180} \pi],\\
\text{the joint of foot:}&[-\frac{1}{180} \pi,\frac{1}{180}],[-\frac{45}{180} \pi,\frac{45}{180} \pi],[-\frac{30}{180} \pi,\frac{5}{180} \pi].
\endaligned
\end{equation*}

\texttt{3D Humanoid}: We refer to the standard \texttt{3D Humanoid}~\cite{tassa2012synthesis} and expand the number of actuators. 
The rotation range of newly-added joint actuators are limited to $[-\frac{1}{180} \pi, \frac{1}{180} \pi]$.

\texttt{3D Cheetah}: 
The standard half-cheetah~\cite{wawrzynski2007learning,wawrzynski2009cat}  is specially designed as a planar model of a walking animal, which would not fall over in 2D-Planar environments, so there is no interruption in each episode. But in 3D-SGRL environments, the half-cheetah very easy to falls over and this will interrupt its learning process, making it more difficult for effective locomotion.
So we modify the model of a half-cheetah into a full-cheetah, and its torso, four legs and tail are made of 14 limbs. 
\texttt{3D Cheetah} is about 1.1 meters long, 0.6 meters high and weighs 55kg. We limit the ``strengths'' of its joints within the range from 30 to 120Nm. So it is designed as a 3D model of a large and agile cat with many joints yet smaller strength, making it more stable and less easy to fall over in 3D-SGRL environments while retaining a strong locomotion ability.
As a result, the full-cheetah is more adaptable to 3D-SGRL environments.
The rotation range of joints is limited to new intervals. 
The rotation range of the tail is $[-\frac{20}{180} \pi, \frac{20}{180} \pi], [-\frac{80}{180} \pi, \frac{80}{180} \pi], [-\frac{1}{180} \pi, \frac{1}{180} \pi]$.
The rotation range of the left limb and the right limb are the same, we only show the intervals of those left:
\begin{equation*}
\aligned
\text{the joint of back thigh:}&[-\frac{10}{180} \pi,\frac{0}{180}],[-\frac{60}{180} \pi,\frac{30}{180} \pi],[-\frac{15}{180} \pi,\frac{5}{180} \pi],\\
\text{the joint of back shin:}&[-\frac{1}{180} \pi,\frac{1}{180}],[-\frac{45}{180} \pi,\frac{45}{180} \pi],[-\frac{1}{180} \pi,\frac{1}{180} \pi],\\
\text{the joint of back foot:}&[-\frac{1}{180} \pi,\frac{1}{180}],[-\frac{45}{180} \pi,\frac{25}{180} \pi],[-\frac{15}{180} \pi,\frac{5}{180} \pi],\\
\text{the joint of front thigh:}&[-\frac{15}{180} \pi,\frac{5}{180}],[-\frac{40}{180} \pi,\frac{60}{180} \pi],[-\frac{20}{180} \pi,\frac{10}{180} \pi],\\
\text{the joint of front shin:}&[-\frac{1}{180} \pi,\frac{1}{180}],[-\frac{50}{180} \pi,\frac{70}{180} \pi],[-\frac{1}{180} \pi,\frac{1}{180} \pi],\\
\text{the joint of front foot:}&[-\frac{1}{180} \pi,\frac{1}{180}],[-\frac{30}{180} \pi,\frac{30}{180} \pi],[-\frac{20}{180} \pi,\frac{5}{180} \pi].
\endaligned
\end{equation*}

To systematically investigate the proposed method applied to multi-task training, we construct several variants from the agents we mentioned above, as shown in \Cref{tab:environments}. 
The morphologies of ten variants of \texttt{3D Cheetah} are different from that of the 2D-Planar, as is shown in \Cref{fig:demo}.

\begin{figure*}[t!]
\centering
\includegraphics[width=\linewidth]{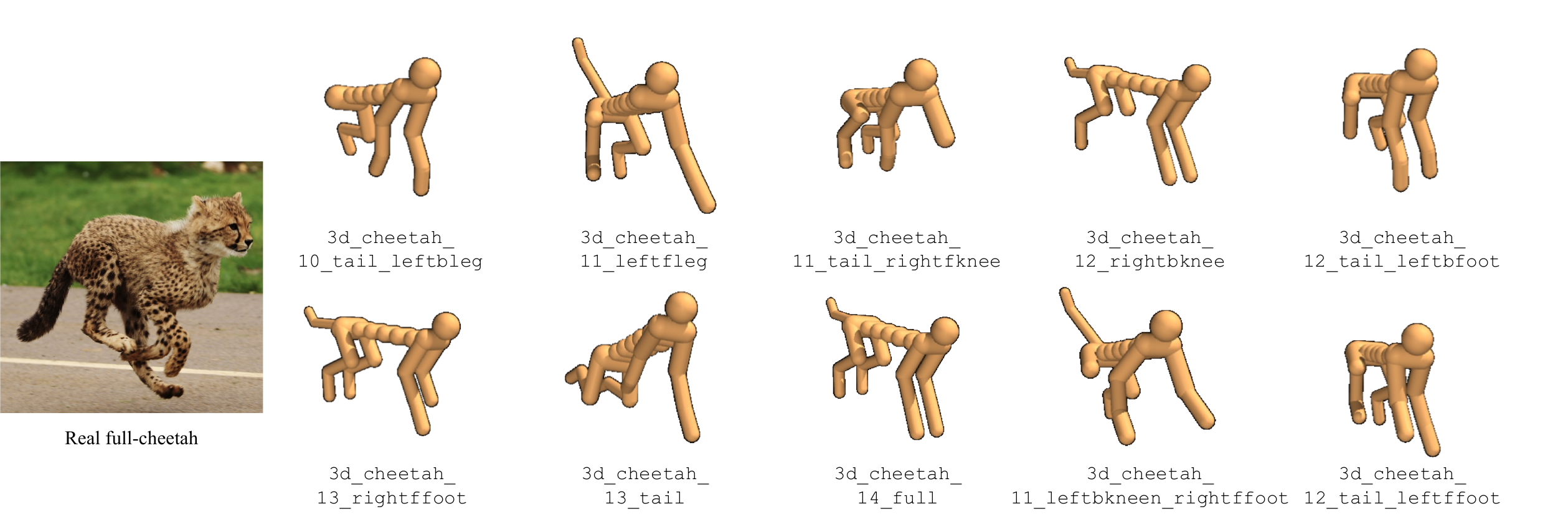}
\vspace{-.2in}
\caption{The morphologies of 10 variants of cheetah. }
\label{fig:demo}
\vspace{-.1in}
\end{figure*}

\begin{table*}
    \centering
    \caption{Full list of environments used in this work.}
    \label{tab:environments}
    \begin{tabular}{lll}
    \toprule
    \textbf{Environment}&\textbf{Training}&\textbf{Zero-shot testing}\\
     \midrule
     \texttt{3D\_Hopper++}&&\\
     \midrule
     &\texttt{3d\_hopper\_3\_shin}&\\
     &\texttt{3d\_hopper\_4\_lower\_shin}&\\
     &\texttt{3d\_hopper\_5\_full}&\\
    \midrule
    \texttt{3D\_Walker++}&&\\
    \midrule
    & \texttt{3d\_walker\_2\_right\_leg\_left\_knee} &\texttt{3d\_walker\_3\_left\_knee\_right\_knee}\\
     & \texttt{3d\_walker\_3\_left\_leg\_right\_foot} &\texttt{3d\_walker\_6\_right\_foot}\\
     & \texttt{3d\_walker\_4\_right\_knee\_left\_foot} \\
     & \texttt{3d\_walker\_5\_foot} \\
     & \texttt{3d\_walker\_5\_left\_knee} &\\
     & \texttt{3d\_walker\_7\_full} &\\
     \midrule
     \texttt{3D\_Humanoid++}&&\\
     \midrule
     & \texttt{3d\_humanoid\_7\_left\_arm} &\texttt{3d\_humanoid\_7\_left\_leg}\\
     & \texttt{3d\_humanoid\_7\_lower\_arms} & \texttt{3d\_humanoid\_8\_right\_knee}\\
     & \texttt{3d\_humanoid\_7\_right\_arm} &\\
     & \texttt{3d\_humanoid\_7\_right\_leg} &\\
     & \texttt{3d\_humanoid\_8\_left\_knee} &\\
     & \texttt{3d\_humanoid\_9\_full} &\\
     \midrule
     \texttt{3D\_Cheetah++}&&\\
     \midrule
     & \texttt{3d\_cheetah\_10\_tail\_leftbleg} & \texttt{3d\_cheetah\_11\_leftbkneen\_rightffoot}\\
     & \texttt{3d\_cheetah\_11\_leftfleg} & \texttt{3d\_cheetah\_12\_tail\_leftffoot}\\
     & \texttt{3d\_cheetah\_11\_tail\_rightfknee} &\\
     & \texttt{3d\_cheetah\_12\_rightbknee} &\\
     & \texttt{3d\_cheetah\_12\_tail\_leftbfoot} &\\
     & \texttt{3d\_cheetah\_13\_rightffoot} &\\
     & \texttt{3d\_cheetah\_13\_tail} &\\
     & \texttt{3d\_cheetah\_14\_full} &\\
     \midrule
     \multicolumn{3}{l}{\texttt{3D\_Walker-3D\_Humanoid-3D\_Hopper++(3D\_WHH++)}}\\
     \midrule
     &\multicolumn{2}{l}{Union of \texttt{3D\_Walker++}, \texttt{3D\_Humanoid++} and \texttt{3D\_Hopper++}}\\
     \midrule
     \multicolumn{3}{l}{\texttt{3D\_Cheetah-3D\_Walker-3D\_Humanoid-3D\_Hopper++(3D\_CWHH++)}}\\
     \midrule
     &\multicolumn{2}{l}{Union of \texttt{3D\_Cheetah++}, \texttt{3D\_Walker++}, \texttt{3D\_Humanoid++} and \texttt{3D\_Hopper++}}\\
    \bottomrule
    
    \end{tabular}
\end{table*}

\subsection{Baselines}
\label{sec:repro} 
This part illustrates the implementations of these baselines.

\paragraph{\SMP{}} \citet{huang2020one} employs GNNs as policy networks and uses both bottom-up and top-down message passing schemes through the links between joints for coordinating. We use the implementation of \SMP{} in the \SWAT{} codebase, which is the same as the original implementation of \SMP{} provided by \citet{huang2020one}.

\paragraph{\SWAT{}} All of the GNN-like works show that morphology-agnostic policies are more advantageous than the monolithic policy in tasks aiming at tackling different morphologies. However, \citet{kurin2020cage} validate a hypothesis that the benefit extracted from morphological structures by GNNs can be offset by their negative effect on message passing.
They further propose a transformer-based method, \Amorpheus{}, which relies on mechanisms for self-attention as a way of message transmission.
\citet{hong2021structure} make use of morphological traits via structural embeddings, enabling direct communication and capitalizing on the structural bias.
We use the original implementation of \SWAT{} released by \citet{hong2021structure}. 
For a fair comparison, \name{} uses the same hyperparameters as \SWAT{} (\Cref{tab:hypers}).

\paragraph{\Monolithic{}} We choose TD3 as the standard monolithic RL baseline.  The actor and critic of TD3 are implemented by fully-connected neural networks.

\subsection{Implementation details}\label{sec:imple}

For the scalar features $\vh_i \in \sR^{13}$,  in addition to retaining the original rotation angle of joint, we also undergo the following processing:
the rotation angle and range of joint are represented as three scalar numbers $(angle_t, low, high)$ normalized to $[0, 1]$, where $angle_t$ is the joint position at time $t$, and $[low, high]$ is the allowed joint range.
The type of limb is a 4-dimensional one-hot vector representing ``torso'', ``thigh'', ``shin'', ``foot'' and ``other'' respectively.
Besides, note that the torso limb has no joint actuator in any of these environments, so we ignore its predicted torque values. 
We implement \name{} based on \SWAT{} codebase~\citep{hong2021structure}, which is built on Official PyTorch Tutorial. \SWAT{} also shares the codebase with \SMP{}~\citep{huang2020one} and \Amorpheus{}~\cite{kurin2020cage}. \Cref{tab:hypers} provides the hyperparameters needed to replicate our experiments. Our codes are available on \href{https://github.com/alpc91/SGRL}{https://github.com/alpc91/SGRL}.

\begin{table}[h]
\centering
\caption{Hyperparameters of our \name{}.}\label{tab:hypers}
\begin{tabular}{@{}cc@{}}
\toprule
\textbf{Hyperparameter}        & \textbf{Value} \\ \midrule
Learning rate                  & 0.0001         \\
Gradient clipping              & 0.1            \\
Normalization                  & LayerNorm      \\
Total attention layers         & 3              \\
Attention heads                & 2              \\
Attention embedding size        & 128            \\
Attention hidden size          & 256            \\
Matrix embedding size          & 32$\times$32           \\
Matrix hidden size          &    512          \\
Encoder output size            & 128            \\
Mini-batch size                & 100            \\
Maximum Replay buffer size             & 10M            \\
 \bottomrule
\end{tabular}
\end{table}

\section{More Discussion about Invariant Methods}
\label{sec:discuss}
Specifically, by choosing the ``forward" direction, we can achieve heading-equivariance with heading normalization. 
In essence, the lack of a predetermined ``forward" direction that is consistent across all agents prevents us from transferring experiences between different agents. For example, if we create a duplicate of one agent and redefine the ``forward" direction, heading normalization will no longer be applicable.
In particular, let's consider two agents that have very similar morphology, with the only difference being that their torso orientations are opposite and both encourage movement along the torso orientation. If the torso orientation is selected as the ``forward" direction, the normalization applied to these two agents will vary significantly. As a result, the policy learned by one agent will not generalize to the other agent, unless the other agent's movement mode is to move in the opposite orientation of the torso. Therefore, generalization performance is affected by the choice of the ``forward" direction and the agent's movement mode. 

Besides, there is extensive experimental evidence~\cite{hsu2022learning, jorgensen2022equivariant, schutt2021equivariant, joshi2022expressive} indicating that equivariant methods that preserve equivariance at each layer outperform those invariant methods that solely apply transformations at the input layer to obtain invariant features and then use an invariant network. Our framework, falling into the equivariant family, enables the propagation of directional information through message passing steps, allowing the extraction of rich geometric information such as angular messages. In contrast, the invariant methods may result in the loss of higher-order correlations between nodes, which are crucial for modeling the geometric relationships between them.

\section{More Ablation on Equivariance}
\label{sec:equivariance}
In addition, we conduct another experiment by fixing the initial orientation as 0° when training, but allowing arbitrary angles when testing. As shown in \Cref{tab:fixinitd}, \name{} generalizes well to all cases. On the contrary, \SWAT{} only obtains desirable performance when the testing angle is fixed to 0° which is the same as that during the training process, and its performance drops rapidly in other cases, especially at 180°. The experiments here justify the efficacy of involving orthogonality equivariance.

\begin{table*}
    \centering
    \caption{Fixed initial orientation (about 0°) training, arbitrary initial orientation (any given angle) test on \texttt{3d\_cheetah\_14\_full}. The table header (the first row of the table) represents the progress of training and the initial orientation. }
    \label{tab:fixinitd}
    \resizebox{\textwidth}{!}{
    \begin{tabular}{ccccccccccc}
    \toprule
     \multirow{2}{*}{Methods}&\multicolumn{5}{c}{500k training steps}&\multicolumn{5}{c}{1M training steps}\\
     & $0^\circ$ & $90^\circ$ & $180^\circ$ & $270^\circ$ & random & $0^\circ$ & $90^\circ$ & $180^\circ$ & $270^\circ$ & random\\
     \midrule
     \SWAT{} & $1886.1\pm148.9$ & $1005.5\pm615.3$ & $\textbf{120.5}\pm178.5$ & $791.0\pm493.4$ & $1232.3\pm72.9$ & $2592.6\pm155.6$ & $1340.2\pm668.0$ & $\textbf{-5.6}\pm8.5$ & $1193.5\pm345.2$ & $1178.6\pm674.9$\\
     \name{} & $1587.4\pm411.3$ & $1695.6\pm278.4$ & $\textbf{1659.9}\pm110.2$ & $1388.3\pm173.8$ & $1465.2\pm161.0$ & $4622.0\pm292.8$ & $4799.5\pm172.9$ & $\textbf{4756.3}\pm103.4$ & $4899.8\pm139.7$ & $4902.8\pm62.9$\\
    \bottomrule
    \end{tabular}}
\vspace{-.1in}
\end{table*}

\section{The Evaluation on v2-variants}
\label{sec:v2}
The v2-variants ($R=10\sim20 m$) are more challenging. We train the policy in the multi-task setting where $R=10km$, then we do the test in v2-variants. The results and related demos are shown in \Cref{fig:hopper_v2}, \Cref{fig:walker_v2}, \Cref{fig:humanoid_v2} and \Cref{fig:cheetah_v2}. While \SWAT{} fails to perform well, \name{} has obvious advantages. With more episode timesteps, \name{} locomotes closer to the destination (a shorter distance) and gets more episode rewards.

\begin{figure}[t!]
\centering
\includegraphics[width=0.92\linewidth]{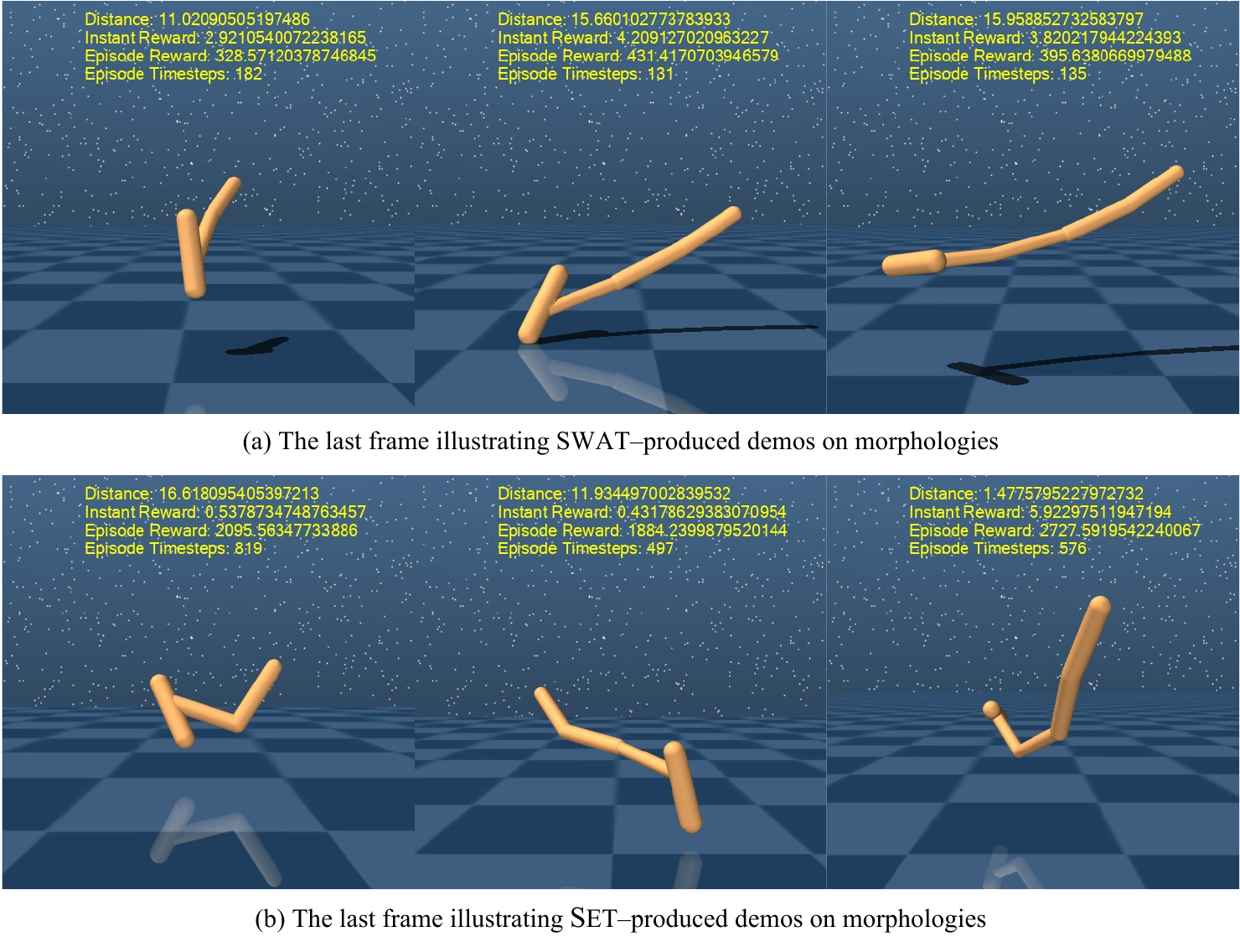}
\vspace{-.22in}
\caption{The evaluation on v2-variants on \texttt{3D\_Hopper++}.}
\label{fig:hopper_v2}
\vspace{-.05in}
\end{figure}

\begin{figure}[t!]
\centering
\includegraphics[width=0.92\linewidth]{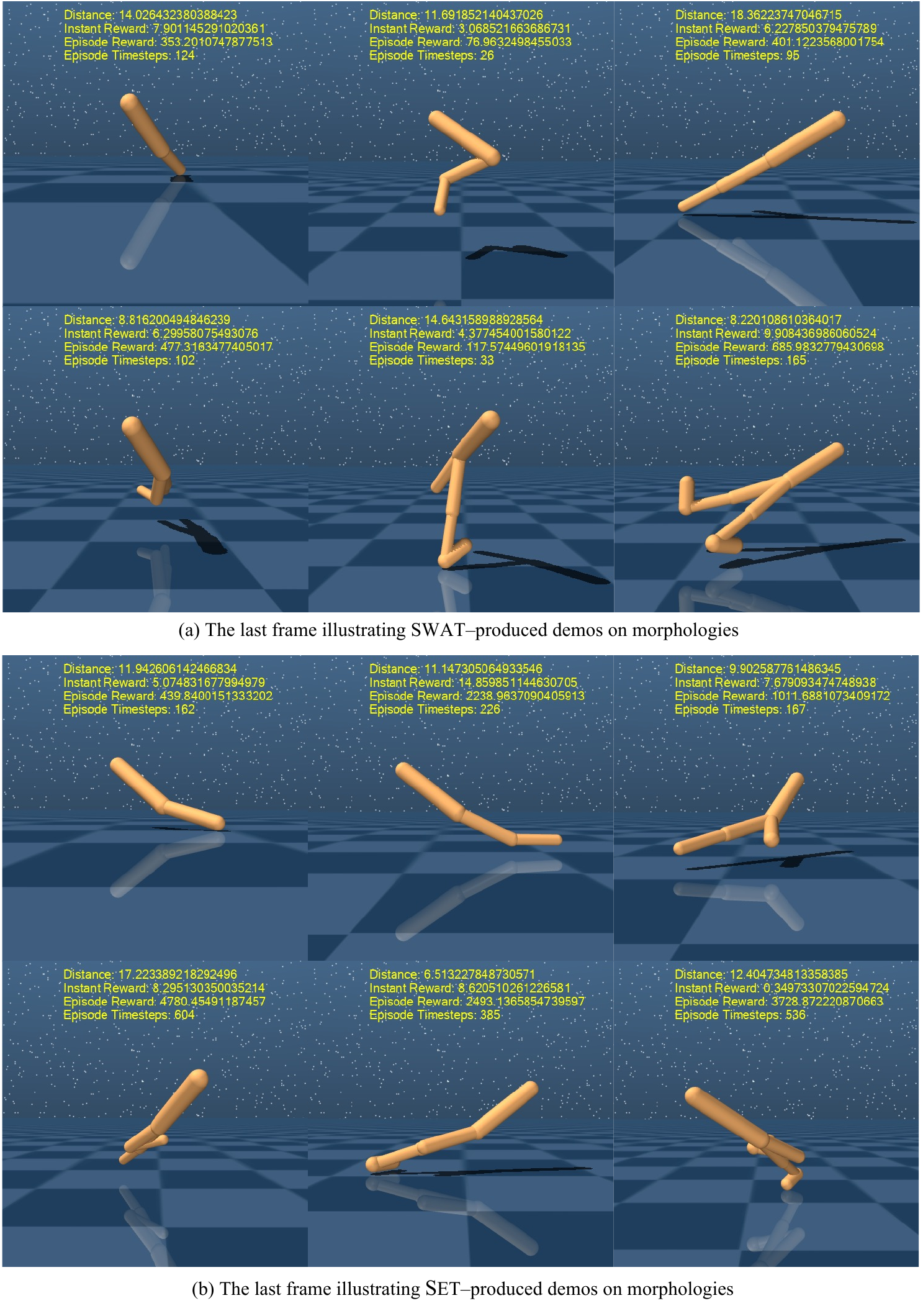}
\vspace{-.22in}
\caption{The evaluation on v2-variants on \texttt{3D\_Walker++}.}
\label{fig:walker_v2}
\vspace{-.05in}
\end{figure}

\begin{figure}[t!]
\centering
\includegraphics[width=0.92\linewidth]{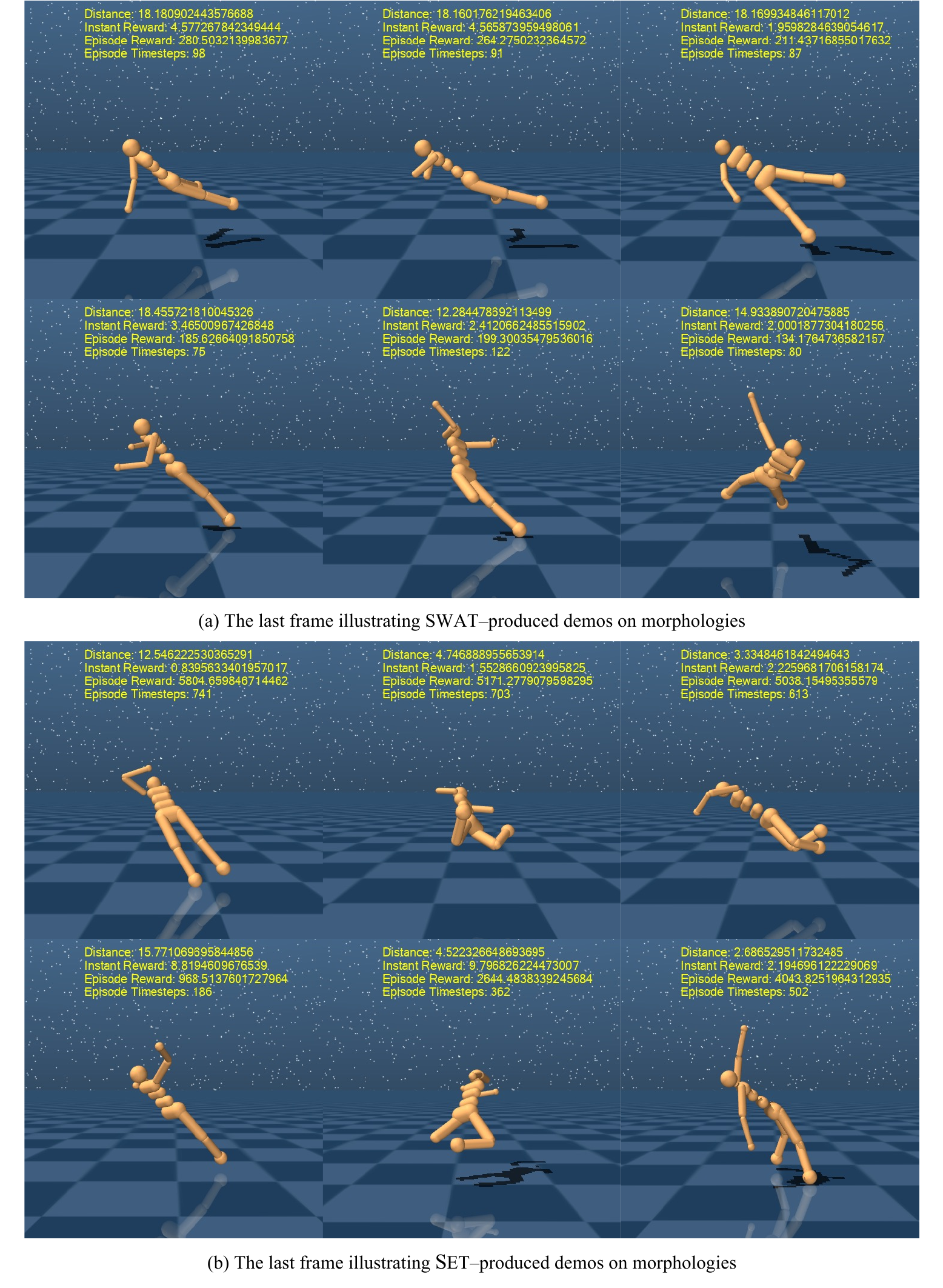}
\vspace{-.22in}
\caption{The evaluation on v2-variants on \texttt{3D\_Humanoid++}.}
\label{fig:humanoid_v2}
\vspace{-.05in}
\end{figure}

\begin{figure}[t!]
\centering
\includegraphics[width=\linewidth]{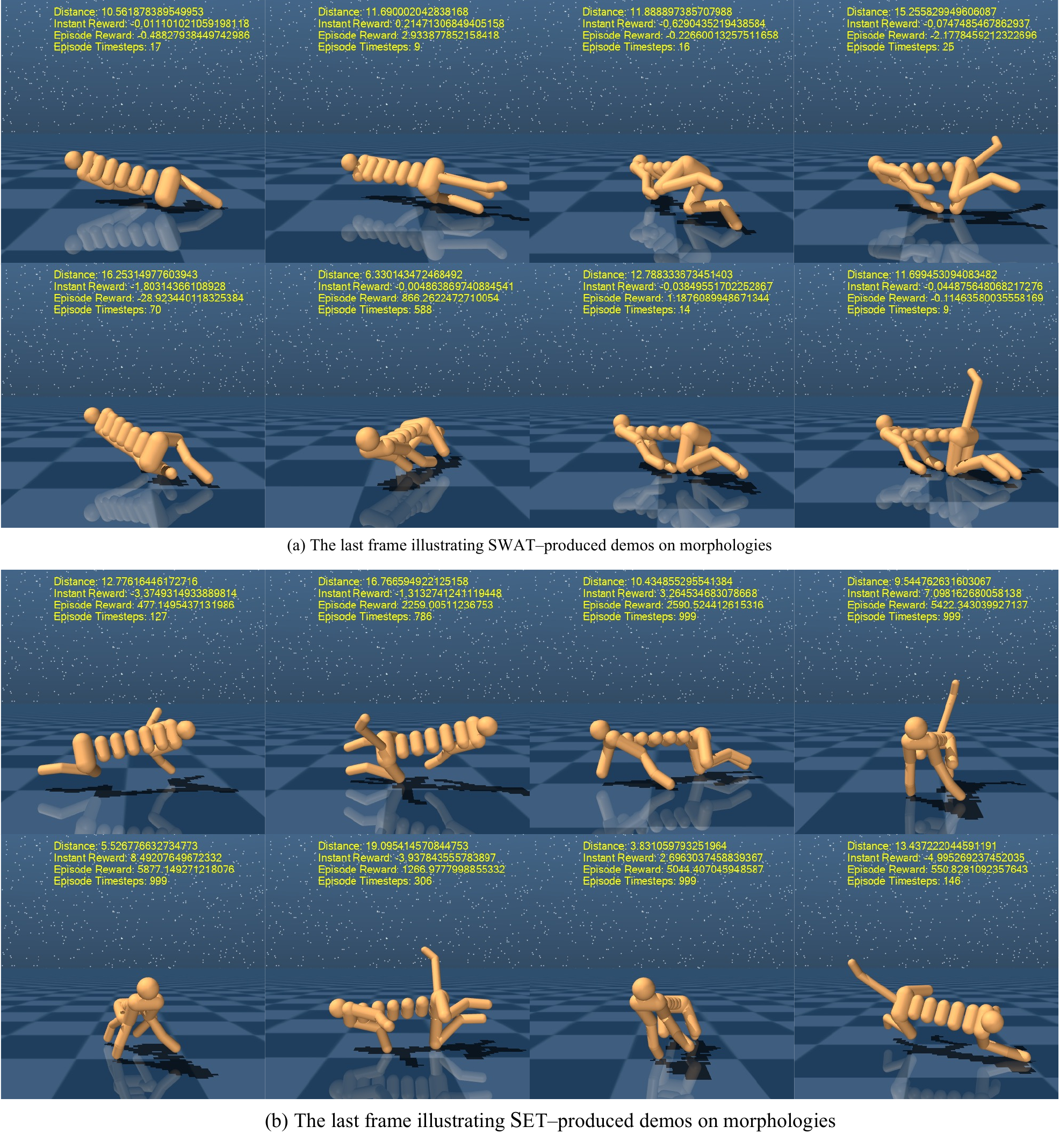}
\vspace{-.32in}
\caption{The evaluation on v2-variants on \texttt{3D\_Cheetah++}.}
\label{fig:cheetah_v2}
\vspace{-.05in}
\end{figure}
%%%%%%%%%%%%%%%%%%%%%%%%%%%%%%%%%%%%%%%%%%%%%%%%%%%%%%%%%%%%%%%%%%%%%%%%%%%%%%%
%%%%%%%%%%%%%%%%%%%%%%%%%%%%%%%%%%%%%%%%%%%%%%%%%%%%%%%%%%%%%%%%%%%%%%%%%%%%%%%

\end{document}